%% file: usenix.tex
\documentclass[letterpaper,twocolumn,10pt]{article}
\usepackage{usenix,epsfig,endnotes}
\usepackage[ruled,vlined,linesnumbered]{algorithm2e}

\SetCommentSty{mycommfont}

\usepackage{algpseudocode}
\usepackage{tikz}
\usepackage{amsmath}
\usepackage{amsthm}
\usepackage{amssymb}
\usepackage{filecontents}
\usepackage{url}
\usepackage{caption}
\usepackage{subcaption}
\usepackage{graphicx}
\usepackage{enumitem}
\usepackage{lipsum}

\usepackage{tabularx}
\usepackage{multirow}
\usepackage{makecell}

\usepackage{hyperref}
\usepackage{xcolor,colortbl}
\begin{document}

\newtheorem{theorem}{Theorem}
\newtheorem{lemma}{Lemma}
\newtheorem{claim}[lemma]{Claim}
\newtheorem{proposition}{Proposition}
\newtheorem{fact}[lemma]{Fact}
\newtheorem{corollary}{Corollary}
\newtheorem{conjecture}{Conjecture}

\newtheorem{definition}{Definition}

\newcommand{\real}{\mathbb{R}}
\newcommand{\ball}{\mathbb{B}}
\newcommand{\expect}{\mathbb{E}}
\date{}

\title{\Large \bf DPAF: Image Synthesis via Differentially Private Aggregation in Forward Phase}

\author{Chih-Hsun Lin, Chia-Yi Hsu, Chia-Mu Yu, Yang Cao, Chun-Ying Huang
}

\maketitle

\thispagestyle{empty}

\begin{abstract}
Differentially private synthetic data is a promising alternative for sensitive data release. Many differentially private generative models have been proposed in the literature. Unfortunately, they all suffer from the low utility of the synthetic data, particularly for images of high resolutions. Here, we propose DPAF, an effective differentially private generative model for high-dimensional image synthesis. Different from the prior private stochastic gradient descent-based methods that add Gaussian noises in the backward phase during the model training, DPAF adds a differentially private feature aggregation in the forward phase, bringing advantages, including the reduction of information loss in gradient clipping and low sensitivity for the aggregation. Moreover, as an improper batch size has an adverse impact on the utility of synthetic data, DPAF also tackles the problem of setting a proper batch size by proposing a novel training strategy that asymmetrically trains different parts of the discriminator. We extensively evaluate different methods on multiple
image datasets (up to images of $128\times 128$ resolution) to demonstrate the performance of DPAF.
\end{abstract}

\section{Introduction}
The advance of deep learning (DL) has witnessed the success of a variety of domains, such as computer vision, natural language, and speech recognition. Training deep neural networks (DNNs) requires a large amount of high-quality data. Unfortunately, lots of valuable data are privacy-sensitive, and the direct release of the data becomes infeasible. Synthetic data has been proposed as a means to overcome the above difficulty. In particular, synthetic data from generative models may possess the same statistical information as the original data and, therefore, leads to great data utility, which means that the analytical result made on synthetic data is similar to the one made on the original data. Synthetic data sampled from the data distribution estimated by generative models is usually assumed to be decoupled from the original data and, therefore, implies privacy. However, recent studies reveal that the generative models still leak privacy information~\cite{gan-leaks, logan, MCMCMIA}. 

Differential privacy (DP)~\cite{dwork2014algorithmic} has been the gloden standard for data privacy. Differentially private deep learning (DPDL), a variant of deep learning with the DP guarantee, can be used to provably preserve the privacy of DL models. For example, Abadi et al.~\cite{Abadi2016DeepLW} propose differentially private stochastic gradient descent (DPSGD) to train a DP classifier by adding the Gaussian noise on the clipped gradient. A majority of subsequent works apply DPSGD during the training of generative adversarial networks (GANs) to derive differentially private generative adversarial networks (DPGANs). Nevertheless, the gradient clipping in DPSGD incurs significant information loss, considerably compromising model accuracy. To enhance the utility of DPDL models, PATE~\cite{papernot2017semisupervised} is proposed as an alternative of DPSGD by considering the aggregation of multiple teacher models trained on the partitioned private data, with the advantage of no information loss. Nonetheless, despite no information loss, applying PATE to generative models is by no means trivial, because the student generator may output synthetic samples with similar labels and, consequently, teacher models cannot learn to be updated effectively~\cite{pate-gan}. 

Though DPSGD remains the most popular method to train a DPDL model, DPSGD has both negative~\cite{Disparate} and positive~\cite{removingDisparate, unlocking} impacts on the model utility. Thus, having a proper design of DPGANs is far from trivial. In this work, we aim to develop a DPGAN for image synthesis by resorting to DPSGD. The primary objective of our work is to generate high-dimensional images in a DP manner such that the downstream classification task can have high classification accuracy. As mentioned above, the utility of DPSGD-based DPGANs is degraded due to two factors, the information loss in gradient clipping and DP noise. In this sense, our proposed techniques are designed to reduce both the information loss (by minimizing the size of the gradient vector) in gradient clipping and DP noise (by minimizing the global sensitivity). 

Several efforts have been made to reduce the information loss of DPSGD. For example, GS-WGAN~\cite{gs-wgan} minimizes the information loss by taking advantage of the 1-Lipschitz property of WGAN~\cite{wgan}. However, its generative capability is also limited by the WGAN structure. DataLens~\cite{datalens} compresses the gradient vector by keeping only the top-$k$ values and then quantizing them. As the non-top-$k$ values in the gradient vector do not involve gradient clipping, more informative magnitudes from top-$k$ values can be kept. Nonetheless, the success of DataLens heavily relies on training a large number of teacher classifiers, hindering the practicality due to large amounts of GPU memory. 

Our method takes a fundamentally different approach; i.e., we conduct a DP feature aggregation in the forward phase during the training of NN. The post-processing of DP ensures that the number of dimensions of the gradient vector can be effectively reduced. On the other hand, choosing a large batch size is a straightforward manner to reduce the negative impact of DP noise because the DP noise may be easier to cancel each other out. Unfortunately, a large batch size leads to much fewer updates of the discriminator and thus the training cannot converge given a fixed number of epochs. In addition, because of the feature aggregation in our proposed method, the large batch size also easily makes features indistinguishable, which is harmful to the synthetic data utility and, in turn, favors a small batch size instead. We tackle this dilemma by conducting a crafted design of training strategy. 

\noindent\textbf{Contribution.} The contributions are summarized below.
\begin{itemize}[leftmargin=*]

    \item We propose DPAF (\textbf{D}ifferentially \textbf{P}rivate \textbf{A}ggregation in \textbf{F}orward phase), an effective generative model for differentially private image synthesis. DPAF supports the conditional generation of high-dimensional images. 

    \item We propose a novel framework to enforce DP during the GAN training. In particular, we propose to place a differentially private feature aggregation (DPAGG) in the forward phase. Together with a simplified instance normalization (SIN), DPAGG can not only have a natural and low global sensitivity but also significantly reduce the dimensionality of the gradient vector. We also have a novel design of the asymmetric model training, resolving the dilemma that a small batch cannot effectively reduce the DP noise but a large batch will make features indistinguishable. 

    \item We formally prove the privacy guarantee of DPAF. Furthermore, we conduct extensive experiments on DPAF with four popular image datasets, including MNIST ($28\times 28$), Fashion-MNIST ($28\times 28$), CelebA (rescaled to $64\times 64$), and FFHQ (rescaled to $128\times 128$). Our experiment results show that DPAF outperforms the prior solutions on CelebA and FFHQ in terms of both classification accuracy (e.g.,14.57\% improvement on CelebA-Gender with privacy budget $\varepsilon=1$ compared to SOTA) and visual quality, as DPAF is featured by its unique characteristic that the generative capability will be maximized for larger images. 
    
\end{itemize}

\section{Related Work}\label{sec: Related Work}
This paper primarily studies how to synthesize images with the DP guarantee. The works for DP tabular data synthesis~\cite{zhang2021privsyn, Mckenna2019PGM, privbayes} are not in our consideration. Though DP generative models and DP classifiers~\cite{unlocking, Tramr2020DifferentiallyPL, Large-scale-private-learning, yu2022differentially, Large-Language-Models} share DPSGD-based training strategies, the design of DP classifiers is beyond of scope in this paper. 

\vspace{-0.2cm}\paragraph{Differentially Private Generative Models.} Differentially private stochastic gradient descent (DPSGD)~\cite{Abadi2016DeepLW} and private aggregation of teacher ensembles (PATE)~\cite{papernot2017semisupervised,papernot2018scalable} are two common techniques that achieve DP generative models. The former trains a model by injecting DP noise into the gradient in each training iteration. More specifically, the gradient is first clipped to ensure a controllable global sensitivity and then perturbed by a proper Gaussian noise. The latter partitions the sensitive dataset into disjoint subsets and trains a teacher model for each subset. The noisy aggregated teachers vote to determine the class for the unlabeled public data.

\vspace{-0.2cm}\paragraph{DPSGD-Based Approach.} Gradient clipping, though beneficial to reducing global sensitivity, leads to a dramatic information loss and, therefore, hurts the model utility. Thus, a majority of works focus on improving DPSGD. Zhang et al.~\cite{zhang2018} cluster the parameters with similar clipping bounds and add DP noise to different parts of the gradient. To have a better control of the noise scale, McMahan et al.~\cite{McMahan2018AGA} and Thakkar et al.~\cite{Thakkar2019DifferentiallyPL} calculate the clipping bound via adaptive clipping. GANobfuscator~\cite{Xu2019GANobfuscator}, GS-WGAN~\cite{gs-wgan}, and Xie et al.'s method~\cite{xie2018}, all built upon the improved
WGAN framework~\cite{wgan}, reduce the noise scale by capitalizing on the Lipschitz property. Moreover, GS-WGAN argues that because only the generator will be released, the discriminator does not need to be trained via DPSGD. By projecting the gradients onto a predefined subspace~\cite{Yu2021DoNL,nasr2020fakegradient}, one can reduce the global sensitivity to lower the DP noise scale. 

\vspace{-0.2cm}\paragraph{PATE-Based Approach.} Through the teachers ensemble, PATE-GAN~\cite{pate-gan} uses the noisy labels on the samples generated by the generator to train both the generator and discriminator based on the PATE framework. However, an inherent assumption behind PATE-GAN is that the generator is able to generate samples from the entire real sample space, which is not always valid. G-PATE~\cite{g-pate} achieves DP on the aggregated gradient from teachers ensemble when backpropagating to update the generator, through the Confident-GNMax aggregator~\cite{papernot2018scalable}. 
Similar to G-PATE~\cite{g-pate}, DataLens~\cite{datalens} enforces DP on the aggregated gradient from teachers ensemble, but the aggregated gradient will be compressed and quantized to limit the information loss of gradient clipping. 

\vspace{-0.2cm}\paragraph{Non-Adversarial Learning-based Approach.} Adversarial learning (e.g., GAN) has a notorious reputation on training instability. Thus, some methods do not rely on adversarial learning. For example, DP-MERF~\cite{harder2021dpmerf} trains the generator by minimizing the maximum mean discrepancy (MMD) between the data distribution and generator distribution. By following a framework similar to DP-MERF, DP-Sinkhorn~\cite{dp-sinkhorn} and PEARL~\cite{pearl} train the generator by minimizing the Sinkhorn divergence with semi-debiased Sinkhorn loss and by minimizing the characteristic function distance, respectively. P3GM~\cite{p3gm} considers a two-phase training for the variational autoencoder (VAE). In particular, P3GM trains the encoder first and then trains the decoder with the frozen encoder. Such a two-phase training increases the robustness against the noise for DP. Similar to P3GM, DP$^2$-VAE~\cite{dp2-vae} and DPD-fVAE~\cite{dpd-fvae} are also VAE-based approaches. DPGEN~\cite{dpgen} takes a different approach; it applies randomized response (RR)~\cite{dwork2014algorithmic} to the movement direction of the Markov Chain Monte Carlo, avoiding the information loss of gradient clipping in DPSGD. Unfortunately, it supports only unconditional generation and, consequently, labeling the synthetic data consumes extra privacy. Very recently, DPDC~\cite{dpdc} is proposed to privately synthesize data by taking advantage of the dataset condensation. The simplest form of DPDC is to apply Gaussian noise to the aggregated gradient, which is the sum of the gradient of random samples from a specific class. Due to not only the sequential denoising behavior but also the non-adversarial learning, DPDM~\cite{dpdm}, built upon the diffusion model, is more robust to the DP noise. 


\section{Preliminary}\label{sec: Preliminary}
Here, we introduce some necessary technical backgrounds for DPAF. First, we formulate DP. Then, we briefly describe generative models and transfer learning. Afterward, we will introduce how to train a neural network in a DP manner. 

\vspace{-0.2cm}\paragraph{Differential Privacy.} $(\varepsilon, \delta)$-differential privacy~\cite{dwork2014algorithmic}, $(\varepsilon, \delta)$-DP, is the \textit{de facto} standard for data privacy. The privacy of $(\varepsilon, \delta)$-DP comes from restricting the contribution of an input sample to the output distribution. More specifically, $\varepsilon>0$ bounds the log-likelihood ratio of any particular output when the algorithm is executed on two datasets differing in a sample, while $\delta\in [0, 1]$ is a probability that certain outputs violate the above bound. $(\varepsilon, \delta)$-DP is formulated as follows.

\begin{definition} An algorithm $\mathcal{M}$ is $(\varepsilon, \delta)$-DP if for all $\mathcal{S}\subseteq \text{Range}(\mathcal{M})$ and for any neighboring datasets $\mathcal{D}$ and $\mathcal{D}'$,
\begin{align}\label{eq: DP}
\text{Pr}[\mathcal{M}(\mathcal{D})\in \mathcal{S}]\leq e^{\varepsilon} \text{Pr}[\mathcal{M}(\mathcal{D}')\in \mathcal{S}]+\delta.
\end{align}\label{def: DP}
\vspace{-0.7cm}
\end{definition}
Throughout this paper, $\mathcal{D}$ and $\mathcal{D}'$ are neighboring if $\mathcal{D}$ can be obtained by adding or removing one sample from $\mathcal{D}'$ in the unbounded DP sense \cite{nofreelunch}. DP can quantify privacy loss; i.e., $\varepsilon$ in Eq. (\ref{eq: DP}), called \textit{privacy budget}, measures the \textit{privacy loss}. In essence, privacy goes worse with an accumulated $\varepsilon$. DP can be achieved by applying Gaussian mechanism $G_\sigma$, where the zero-mean Gaussian noise with a proper variance is added to the algorithm output. More specifically, given a function $f$, $G_\sigma\circ f(x)\triangleq f(x)+N(0, \sigma^2)$ obeys $(\varepsilon, \delta)$-DP for all $\varepsilon<1$ and $\sigma>\sqrt{2\ln 1.25}\Delta_{2,f}/\varepsilon$, where $f$'s $\ell_2$-sensitivity is defined as $\Delta_{2, f}\triangleq \max_{\mathcal{D}, \mathcal{D}'}||f(\mathcal{D})-f(\mathcal{D}')||_2$ for neighboring $\mathcal{D}$ and $\mathcal{D}'$.

DP has the following useful properties. First, repeatedly accessing sensitive data leads to an accumulation of privacy loss. Such a privacy loss can be bounded by the sequential composition theorem or higher-order techniques (e.g., moments accountant~\cite{connect-the-dots, Abadi2016DeepLW}). Second, DP is not affected by post-processing. Formally, $g\circ \mathcal{M}$ for any data-independent mapping $g$ still satisfies $(\varepsilon, \delta)$-DP, given that $\mathcal{M}$ is $(\varepsilon, \delta)$-DP.

\vspace{-0.2cm}\paragraph{Generative Adversarial Network.} While generative models are used to derive the data distribution behind the dataset, generative adversarial network (GAN) is the most popular generative model in the DL literature. We may derive the data distribution through GANs and then perform the sampling to obtain a synthetic dataset in a non-private setting.

GAN consists of two components, a generator $G$ and a discriminator $D$. In particular, $G$ learns to output synthetic samples, while $D$ taking as inputs both the synthetic samples and the sensitive dataset is trained to tell real samples apart from the fake ones. Formally, given the real sample $x$ and a sampled noise $z$, $D$ is trained with the loss function $\mathcal{L}_D=-\log D(x)-\log(1-D(G(z)))$. On the other hand, $G$ is trained with the loss function $\mathcal{L}_G=-\log D(g(z))$.

GAN can be extended to conditional GAN (cGAN)~\cite{cgan}. The difference between GAN and cGAN is that the latter is able to generate samples with a specific class (aka. conditional generation), while the former cannot. cGAN can also be instantiated by the above generator-discriminator architecture. In addition, the class label will usually be added to either the generator or discriminator (or both) to ensure that the synthetic sample is consistent with the input class label. Examples of cGANs include AC-GAN~\cite{acgan} and cDCGAN~\cite{cdcgan, cgan}. 

\vspace{-0.2cm}\paragraph{Transfer Learning.} Usually, a DNN consists of two parts, feature extractor (FE) and label predictor. Transfer learning aims to effortlessly train a target model for a specific task by leveraging the pre-trained FE for another but related task. In particular, the target model is composed of the pre-trained FE and a randomly initialized label predictor. During the training of the target model, the FE is frozen and only the label predictor is updated with the training set from the target task. Formally, let $M(x)$ be the function of the target model with $x$ as the input vector. $M$ can be represented as $M(x) = M_c(M_f(x, \theta_f), \theta_c)$, where $M_f$, $M_c$, $\theta_f$, and $\theta_c$ denote the FE, the label predictor, the parameters of the FE, and the parameters of the label predictor, respectively. A typical approach to transfer learning is to optimize the objective, $\min_{\theta_c}\sum_{x\in X,y\in Y} L(M_c(M_f(x, \theta_{f0}), \theta_c), y)$, where $\theta_{f0}$ denotes the parameters of the pre-trained FE, $L$ is a loss function, and $X$ and $Y$ are the training data and labels, respectively. 

\vspace{-0.2cm}\paragraph{Differentially Private Stochastic Gradient Descent.} Differentially private stochastic gradient descent (DPSGD) is the most popular technique to train a DPDL model. Given a training set $\{(x_i, y_i)\}_{i=1}^N$, the update the ordinary stochastic gradient descent (SGD) is formulated as $w^{t+1}=w^{(t)}-\eta_t \frac{1}{B}\sum_{i\in \mathcal{B}_t}\nabla \mathcal{L}(w^{(t)}, x_i, y_i)$, where $\mathcal{L}(w, x, y)$ is the loss function with the model parameter $w$, input sample $x$, and label $y$, and $\mathcal{B}_t$ is the set of samples at iteration $t$ with $\mathcal{B}_t=B$. Because the gradient has an unbounded sensitivity, we have to clip the gradient to ensure an bounded DP noise magnitude. Formally, the update of DPSGD can be formulated as follows.
\begin{small}
\begin{align}
w^{t+1}=w^{(t)}-\eta_t\left\{ \frac{1}{B}\sum_{i\in \mathcal{B}_t}\text{clip}_u\left( \nabla \mathcal{L}(w^{(t)}, x_i, y_i) \right)+\frac{\sigma u}{B}\xi \right\},
\label{eq: DPSGD}
\end{align}\end{small} where $\xi$ is sampled from the zero-mean Gaussian distribution, $\sigma$ specifies the standard deviation of the added noise, and $\text{clip}_u$ is defined as $\text{clip}_u(v)=\min\{1, \frac{u}{||v||_2}\} \cdot v$ with $u$ as a manually configurated clipping threshold. 

\section{Threat Model}
In practice, DL models are usually trained with privacy-sensitive data. Thus, keeping the training data private is a major concern before the real-world deployment of DL-based systems. Unfortunately, given access to the classifiers only, the attacker can still infer whether a specific sample is in training set by launching a membership inference attack (MIA)~\cite{mia, enhanced-mia}. It has been proven that MIA can apply to not only classifiers~\cite{mia} but also GANs~\cite{gan-leaks, logan, MCMCMIA}. Furthermore, compared to MIAs, model inversion attacks~\cite{model-inversion-attack} aim to recover the training samples. Though certain defenses have been developed, the reconstruction attack can still work~\cite{isprivatelearning}.

\textbf{Goal.} DP can protect against MIAs and training data memorization. Thus, our goal is to ensure the DP guarantee of the GANs, while preserving high utility. More specifically, we assume that only the generator of a GAN will be released for data synthesis. In other words, the attacker does not have access to the discriminator. An attacker's access to the generator should not lead to the privacy leakage of training samples, even though the attacker launches MIAs and the above data recovery attacks. Notably, keeping high utility while ensuring privacy in DPGANs is highly challenging, because the DP noise may dramatically hinder the training of DPGANs. Here, the utility in our consideration includes the predicting accuracy of the classifier trained by synthetic samples and tested by real samples and the visual quality of synthetic samples.

\section{Proposed Method}\label{sec: Proposed Method}
\subsection{Overview}\label{sec: Overview}
In contrast to the prior DPGANs, DPAF has a fundamentally different design, where a DP feature aggregation is performed in the forward phase. The aggregated feature makes the image features more robust against the DP noise. On the other hand, the DP feature aggregation in the forward phase implies a shortened gradient vector, resulting in a significant reduction of information loss in gradient clipping. DPAF is also featured by the use of a simplified instance normalization, which preserves fine-grained features and reduces $\ell_2$-sensitivity. Overall, the five advantages of conducting DP feature aggregation in the forward phase are summarized below. 

\begin{figure}
     \centering
     \begin{subfigure}[b]{0.19\textwidth}
         \centering
         \includegraphics[width=\textwidth]{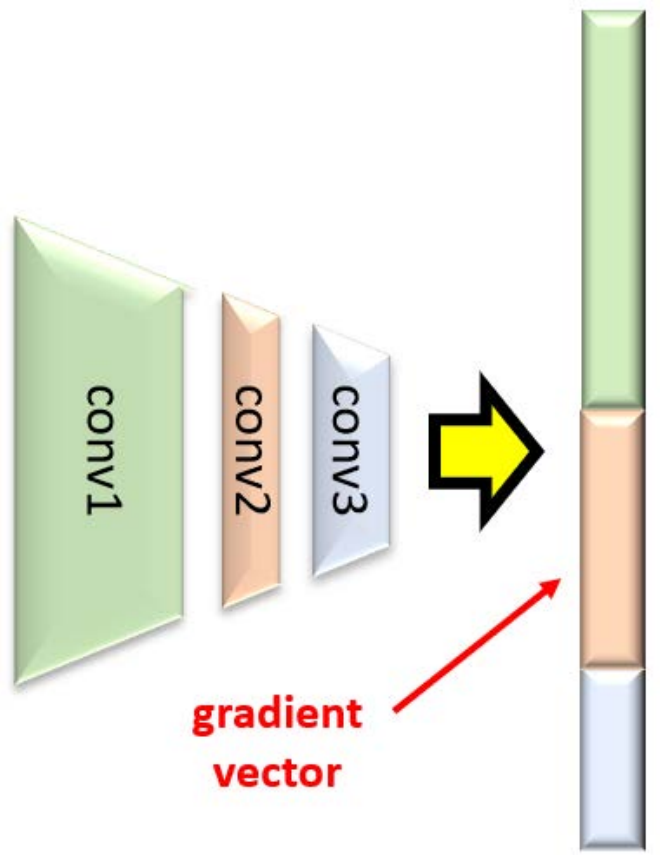}
         \caption{w/o DP aggregation.}
         \label{fig: idea1a}
     \end{subfigure}
     \hspace{0.5cm}
     \begin{subfigure}[b]{0.19\textwidth}
         \centering
         \includegraphics[width=\textwidth]{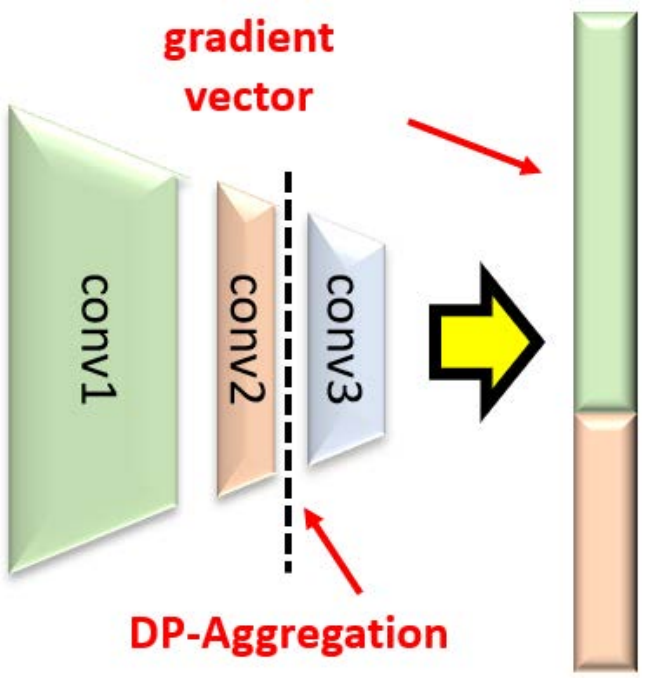}
         \caption{w/ DP aggregation.}
         \label{fig: idea1b}
     \end{subfigure}
     \caption{The illustration of the impact of DP feature aggregation on the size of the gradient vector.}\label{fig: idea1}
\end{figure}

\vspace{-0.2cm}\paragraph{(1) Reduction of Information Loss in Gradient Clipping.} The first advantage is the dimensionality reduction of the gradient vector, as illustrated in Figure~\ref{fig: idea1}. More specifically, gradient clipping in DPSGD inevitably leads to information loss. However, gradient clipping has less impact on shorter gradient vectors, resulting in less information loss. As shown in Figure~\ref{fig: idea1a}, the gradient vector that needs to be sanitized will be lengthier if the DP feature aggregation is not used. On the contrary, as shown in Figure~\ref{fig: idea1b}, conv3 has been privatized after the DP feature aggregation due to the post-processing property of DP, and can be updated by SGD. As a result, because only conv1 and conv2 need to be updated through DPSGD, the information loss from gradient clipping can be mitigated. 

\begin{figure}
     \centering
     \begin{subfigure}[b]{0.19\textwidth}
         \centering
         \includegraphics[width=\textwidth]{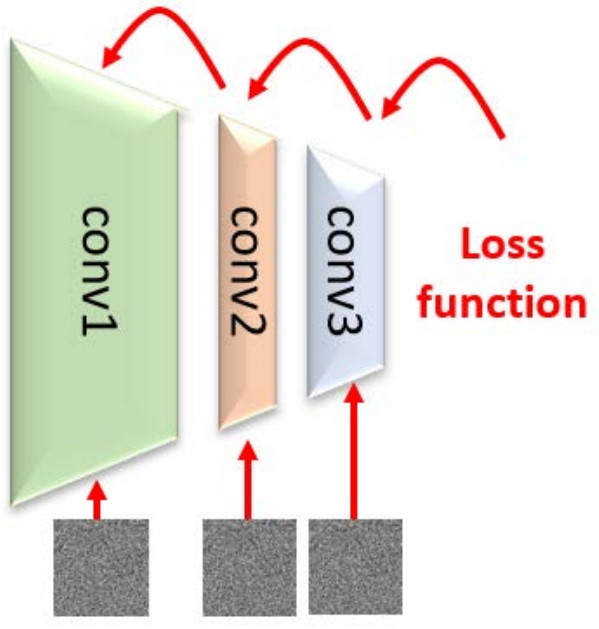}
         \caption{w/o DP aggregation.}
         \label{fig: idea2a}
     \end{subfigure}
     \hspace{0.5cm}
     \begin{subfigure}[b]{0.19\textwidth}
         \centering
         \includegraphics[width=\textwidth]{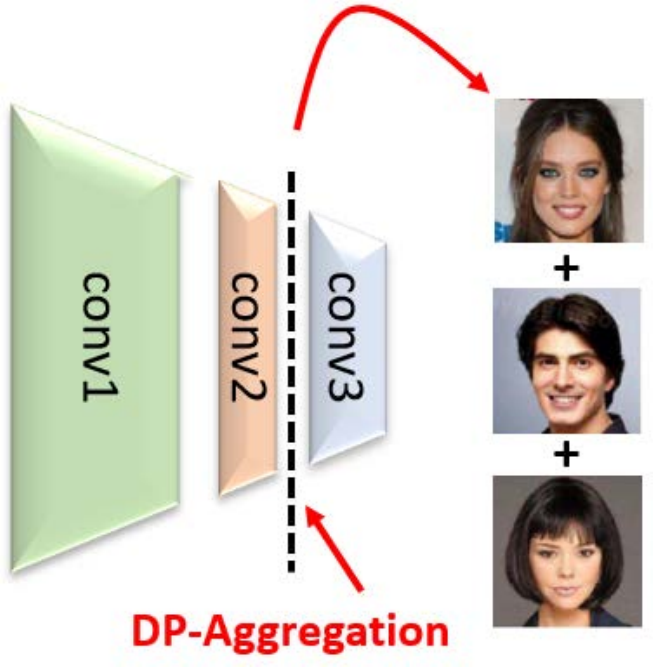}
         \caption{w/ DP aggregation.}
         \label{fig: idea2b}
     \end{subfigure}
     \caption{The impact of DP feature aggregation on the gradient structure preservation during the backpropation.}\label{fig: idea2}
\end{figure}

\vspace{-0.2cm}\paragraph{(2) Better Preserving Gradient Structure.} During the backpropagation, DPSGD applies noise to weights in a layer-by-layer manner, as shown in Figure~\ref{fig: idea2a}, which makes the training harder because such an update process destroys the inherent structure of the gradient vector. Thus, the second advantage is to better preserve the gradient structure. This can be attributed to the fact that the DP aggregated vector is a vector of aggregated noisy image features. As shown in Figure~\ref{fig: idea2b}, since the aggregated features still have the inherent semantics, the corresponding noisy version remains meaningful. On the other hand, also as shown in Figure~\ref{fig: idea2b}, conv3 can be updated by SGD, rather than DPSGD, better preserving the inherent gradient structure. Though still only layers (e.g., conv2) need to be updated by DPSGD, and hence only a small fraction of parameter structure will be affected by DPSGD. 

\begin{figure}[t]
\includegraphics[width=8cm]{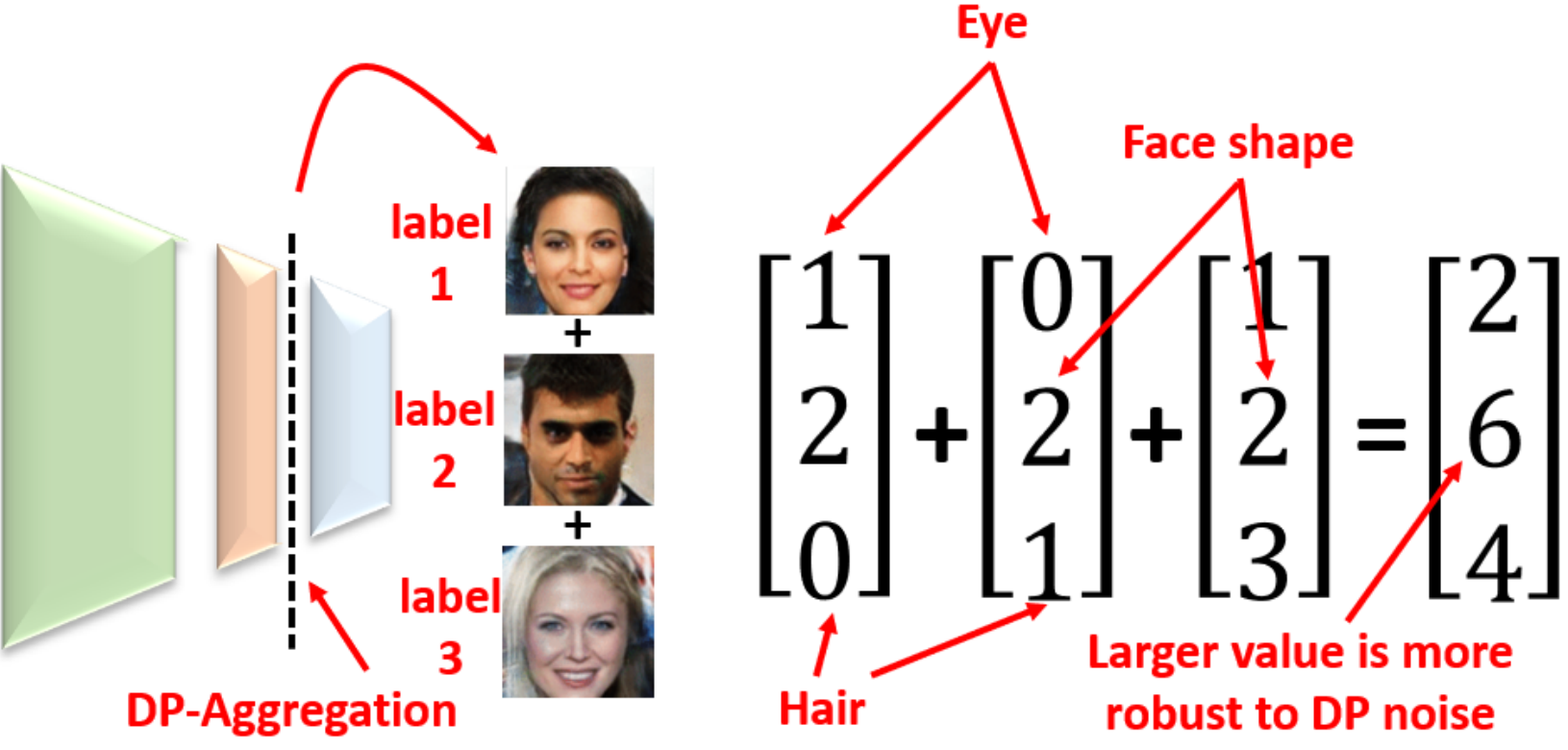}
\centering
\caption{The illustration of the better robustness against the DP noise after the feature aggregation.}
\label{fig: idea3}
\end{figure}

\vspace{-0.2cm}\paragraph{(3) Better Robustness against DP Noise.} The aggregation of the features from samples makes the aggregated feature more robust to the DP noise, because the DP noise is added after the feature value summation. This is illustrated in Figure~\ref{fig: idea3}, where each individual feature is relatively susceptible to the DP noise, but the aggregated one has the larger values and, as a result, better robustness.

\vspace{-0.2cm}\paragraph{(4) Better Preserving Image Features.} From Figure~\ref{fig: idea3}, we know that the aggregated feature vector robust to the DP noise helps in synthesizing realistic samples (because being backpropagated to update $G$), but such synthetic samples may be irrelevant to a specific label. In fact, the feature values are supposed to have similar numeric ranges; otherwise, $D$ will pay attention to only the features with large values and ignores the features with small values. As a consequence, $G$ cannot be updated well. We propose to use a simplified instance normalization (SIN) to ensure that fine-grained features can also be learned. More specifically, SIN is individually applied to each feature map. Then, the feature vector (concatenated normalized feature maps) goes through the aggregation. This helps in synthesizing faces with consistent gender. In other words, in general, without SIN, because of the imbalance of feature values, some feature values will be devoured by the other ones, resulting in the disappearance of certain important feature values that are related to the specific class label. 

\vspace{-0.2cm}\paragraph{(5) Low Global Sensitivity.} The fifth advantage is the low $\ell_2$-sensitivity of the SIN-and-aggregation operation. More specifically, as mentioned above, the feature maps need to be normalized and then concatenated as the feature vector before the aggregation. We find that the $\ell_2$-sensitivity of such an SIN-and-aggregation operation can be calculated as a relatively small and controllable value $\sqrt{m}p$, where $m$ is the number of feature maps and the size of the feature map is $p\times p$. The calculation is deferred to Section~\ref{sec: DPAF}. Properly setting $\sqrt{m}p$ effectively reduces the noise magnitude, thereby raising the utility.

\begin{figure}[t]
\includegraphics[width=8cm]{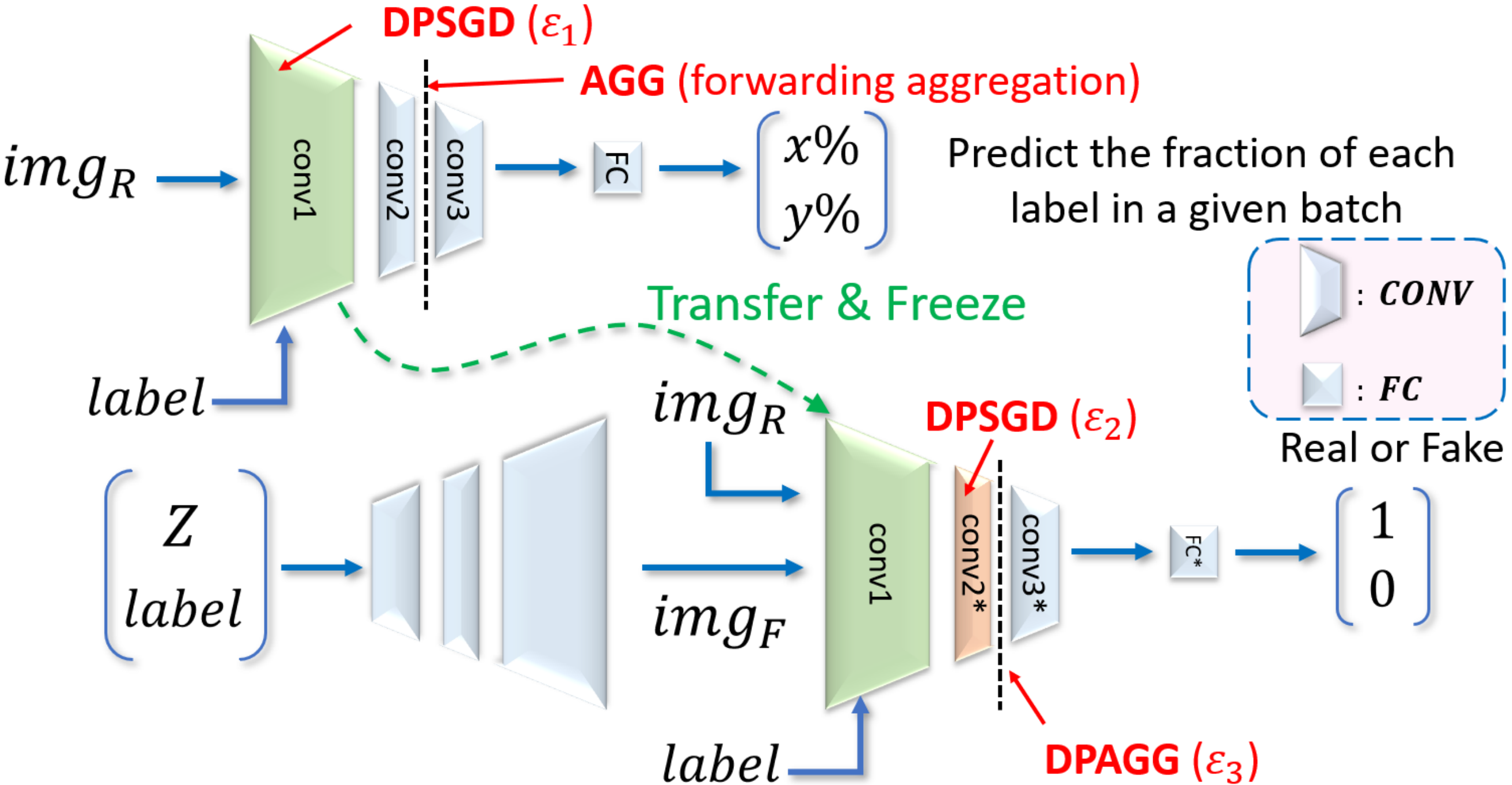}
\centering
\caption{The model architecture of DPAF.}
\label{fig: framework of DPAF}
\end{figure}

\begin{algorithm}[t!]
    \DontPrintSemicolon
	\caption{Training of DPAF}
	\label{algo: Training of DPAF}
	\textbf{Notation}: number of batches $\mathbb{B}$, mean square error loss function $\mathcal{L}_{\text{MSE}}$, binary cross-entropy loss functions $\mathcal{L}_{\text{BCE}}$, $\mathcal{L}'_{\text{BCE}}$, $\mathcal{L}''_{\text{BCE}}$, asymmetry multiplier $\mu$, number of critic iterations per generator iteration $n_{critic}$

    \tcc{the for-loop below trains the classifier $C$}
    \For{$i=1$ to $\mathbb{B}$}
    {
        compute $\mathcal{L}_\text{MSE}$ over
the $i$-th batch 

        SGD for updating conv2, conv3, and FC
        
        DPSGD($\varepsilon_1$) for updating conv1
    }
    conv1*$\leftarrow$ conv1 \tcp*{conv1* and conv1 share parameters}

    \For{$i=1$ to $\mathbb{B}$}
    {
        \tcc{computing $\mathcal{L}_\text{BCE}$ on $D$ with DPAGG($\varepsilon_3$),  $\mathcal{L}_\text{BCE}^{\prime}$ on $D$ that replaces DPAGG($\varepsilon_3$) by AGG, and $\mathcal{L}_\text{BCE}^{\prime\prime}$ on $D$ without DPAGG($\varepsilon_3$), respectively}
        compute $\mathcal{L}_\text{BCE}$ over the $i$-th batch \\
        \tcc{the code below asymmetrically trains $D$}
        \If{$i\%\mu= 0$}
       {
        compute $\mathcal{L}_\text{BCE}^{\prime}$ over the $\left[ i-\mu+1, i \right]$-th batches\\
        SGD for updating conv3* and FC* by $\mathcal{L}_\text{BCE}$\\
        DPSGD($\varepsilon_2$) for updating conv2* by $\mathcal{L}_\text{BCE}^{\prime}$
        }
        \Else
        {
         SGD for updating conv3* and FC* by $\mathcal{L}_\text{BCE}$
        }
        \tcc{the code below trains $G$}
        \If{$i\%n_{critic}= 0$}
        {
            compute $\mathcal{L}_\text{BCE}^{\prime\prime}$ over each sample from the $i$-th batch \\
            SGD for update of $G$
        }
    }
\end{algorithm}

\subsection{DPAF}\label{sec: DPAF}
Here, we present DPAF (\textbf{D}ifferentially \textbf{P}rivate \textbf{A}ggregation in \textbf{F}orward phase), an effective generative model for differentially private image synthesis. The architecture of DPAF is illustrated in Figure~\ref{fig: framework of DPAF}. The notation table summarizing the frequently used notations can be found in Appendix. DPAF is trained by using transfer learning. Hence, similar to transfer learning, DPAF has two phases, training a classifier first and training a GAN, both in the DP manner. We describe each of them below. 

Though using transfer learning as a subroutine to boost the utility, DPAF does not use warm start~\cite{zhang2018}, which is a technique for improving the utility by taking advantage of the extra data with similar data distribution, because it is not always able to find such data for arbitrary sensitive datasets. 

\subsubsection{Training a Classifier Before Transfer Learning}\label{sce: Training a Classifier Before Transfer Learning} The architecture of the classifier $C$ in DPAF before transfer learning is shown in the upper part of Figure~\ref{fig: framework of DPAF}. The classifier $C$ is identical to the discriminator of cDCGAN (in fact, a standard convolutional neural network, CNN), where there are three parts of convolutional layers (conv1, conv2, and conv3) and certain fully connected (FC) layers, except that a layer of aggregation is added between conv2 and conv3. The actual numbers of layers for the convolutional and FC layers and the number of neurons in the input layer are dependent on the input image size. We follow the common setting that the number of neurons in the next convolutional layer is half of the one in the current convolutional layer.

The ordinary CNN is designed to classify inputs. However, because we introduce an aggregation layer (AGG) between conv2 and conv3, which aggregates multiple normalized features in a batch, obviously the CNN can no longer output the predicted class for a given input image and label. Instead, $C$ here is designed to predict the percentage of each class in a given batch, as shown in Figure~\ref{fig: framework of DPAF}. To achieve the above goal, the CNN is trained with the labeled sensitive data (img$_R$ in Figure~\ref{fig: framework of DPAF}) through the mean square error (MSE) loss function, $\mathcal{L}_{\text{MSE}}$. Afterward, SGD is applied to conv2, conv3, and FC for the update of the corresponding parameters, while the DPSGD with the privacy budget $\varepsilon_1$, DPSGD($\varepsilon_1$), is applied to conv1, because only conv1 will be recycled to be used after transfer learning. Lines 2$\sim$5 in Algorithm~\ref{algo: Training of DPAF} correspond to the above training procedures. After training, conv2, conv3, and FC are discarded and will not be released. 

\subsubsection{Training a DPGAN After Transfer Learning}\label{sec: Training a DPGAN After Transfer Learning} In this phase, we aim to train a DPGAN such that a generator satisfying DP can be released. GAN is known to be composed of two parts, a generator $G$ and a discriminator $D$. In DPAF, the architecture of $D$ is identical to $C$, as shown in the bottom part of Figure~\ref{fig: framework of DPAF}. The conv1 in $C$ is transferred to be the conv1 in $D$; i.e., $C$ and $D$ share the same conv1 (Line 6 in Algorithm~\ref{algo: Training of DPAF}). However, conv2*, conv3*, and FC* are randomly initialized. Unlike $C$, where an AGG is placed between conv2 and conv3, a DP feature aggregation layer with the privacy budget $\varepsilon_3$, DPAGG($\varepsilon_3$), is placed between conv2* and conv3* of $D$. The architecture of $G$ is a reverse of $D$ without DPAGG($\varepsilon_3$). 

Training the DPGAN in DPAF is similar to training an ordinary GAN; i.e., we iteratively train $D$ first and then the generator until the convergence. $D$ takes as input the sensitive images (img$_R$ in Figure~\ref{fig: framework of DPAF}), synthetic images (img$_F$ in Figure~\ref{fig: framework of DPAF}), and the label. Given that conv1 is frozen during the training due to transfer learning, $D$ is trained to differentiate between real and synthetic images. The general guideline of training $D$ is that after binary cross entropy (BCE) loss function $\mathcal{L}_{\text{BCE}}$ is calculated on $D$ with DPAGG($\varepsilon_3$), conv3* and FC* can be updated via SGD because such an update of conv3* and FC* still satisfies the DP according to the post-processing of DPAGG($\varepsilon_3$). Consider $\tilde{D}$ as the discriminator $D$ that replaces DPAGG($\varepsilon_3$) by AGG. In addition, another BCE loss function $\mathcal{L}'_{\text{BCE}}$ is calculated on $\tilde{D}$. conv2* will be updated by using DPSGD($\varepsilon_2$) based on $\mathcal{L}'_{\text{BCE}}$. Lines 8$\sim$14 in Algorithm~\ref{algo: Training of DPAF} correspond to the training of $D$. The details about $\mathcal{L}_{\text{BCE}}$ and $\mathcal{L}'_{\text{BCE}}$ are related to our proposed asymmetric training and will be described later.

Consider $\hat{D}$ as the discriminator $D$ without DPAGG($\varepsilon_3$); i.e., $\hat{D}$ can be seen as a standard CNN. After updating $D$ $n_{\text{critic}}$ times, the BCE loss function $\mathcal{L}''_{\text{BCE}}$ is calculated on $\hat{D}$ and is backpropagated to update $G$ through SGD. $n_{\text{critic}}$ is called the number of critic iterations per generator iteration for the better training~\cite{wgan}. Skipping the aggregation when training $G$ can be attributed to the fact that we aim to learn how to generate a single image, instead of a mix of images. Note that $\mathcal{L}_{\text{BCS}}$, $\mathcal{L}'_{\text{BCS}}$, and $\mathcal{L}''_{\text{BCS}}$ all work on the same $D$, but depending on which part of $D$ needs to be updated, different components of $D$ are ignored. Lines 15$\sim$17 in Algorithm~\ref{algo: Training of DPAF} correspond to the training of $G$. 

\vspace{-0.2cm}\paragraph{The Design of DPAGG.} AGG can be implemented via two steps. First, the normalized feature maps are concatenated as a feature vector. Second, the feature vectors from different samples in a batch are aggregated. We can have DPAGG when we apply Gaussian mechanism to the aggregated feature vector derived from AGG. 

Inspired by \cite{instancenormalization, 1607.08022}, we propose to use a simplified instance normalization (SIN) to not only ensure the balance of feature values but also, more importantly, derive a bound of the global sensitivity of AGG. SIN can be formulated as follows. 
\begin{equation}
\begin{split}
\mu_{i_1i_2} = \frac{1}{HW}\sum_{i_3=1}^{H}\sum_{i_4=1}^{W}&x_{i_1i_2i_3i_4}, \sigma_{i_1i_2}^{2} = \frac{1}{HW}\sum_{i_3=1}^{H}\sum_{i_4=1}^{W}\left(x_{i_1i_2i_3i_4} - \mu_{i_1i_2}  \right) ^{2},\\
&\widehat{x_{i_1i_2i_3i_4}} = \frac{x_{i_1i_2i_3i_4} - \mu_{i_1i_2}}{\sqrt{\sigma_{i_1i_2}^{2}}}, \notag
\end{split}
\end{equation}
where $\mu_{i_1i_2}$ is the mean of feature map $X_{i_1i_2}$, $\sigma_{i_1i_2}^{2}$ is the variance of $X_{i_1i_2}$, $i_1$ is the index of the image in the batch, $i_2$ is the feature channel (color channel if the input is an RGB image), $H$ is the height of the feature map, $W$ is the width of the feature map, $x_{i_1i_2i_3i_4} \in \mathbb{R}$ is an element of feature map $X_{i_1i_2}$, $\widehat{x_{i_1i_2i_3i_4}}$ is the new value of 
$x_{i_1i_2i_3i_4}$ after SIN. One can easily see that SIN is different from ordinary instance normalization (IN) in that SIN does not have learnable parameters center and scale~\cite{instancenormalization, 1607.08022}. Concretely, each normalized feature map (through SIN) is guaranteed to have the same $\ell_2$-norm $p$, where $p\times p$ is the size of feature map. 


Before applying Gaussian mechanism to the APP output, we calculate the $\ell_2$-sensitivity $\Delta_{2, \text{AGG}}$ of the AGG below. 
\begin{small}
\begin{align}
\Delta_{2, \text{AGG}} &=\sqrt{ \sum_{x \in X_{11}} (\frac{x - \mu_{11}}{\sigma_{11}})^2 + \cdots +\sum_{x \in X_{1m}} (\frac{x - \mu_{1m}}{\sigma_{1m}})^2} \notag\\
&= \sqrt{ \frac{1}{\sigma_{11}^2} \sum_{x \in X_{11}}({x - \mu_{11}})^2 + \dots + \frac{1}{\sigma_{1m}^2} \sum_{x \in X_{1m}}({x - \mu_{1m}})^2} \notag\\
&=  \sqrt{ \sum_{j=1}^{m}\left[  \frac{1}{\frac{\sum_{x \in X_{1j}}({x - \mu_{1j}})^2}{\left \| X_{1j} \right \|}}\sum_{x \in X_{1j}}({x - \mu_{1j}})^2 \right] } \notag\\
&= \sqrt{ \left \| X_{11} \right \| + \dots + \left \| X_{1m} \right \| } 
= \sqrt{ p^2 + \dots + p^2} 
= \sqrt{m}p,\label{eq: AGG sensitivity}
\end{align}
\end{small}
where $m$ is the number of feature maps, $x$ is an element of feature map $X_{1j}$ for $j=1, 2, \cdots, m$. In the above calculation of $\Delta_{2, \text{AGG}}$, we consider the case where the batch size is $1$ because we aim to know the amount of difference to which a single sample in the batch contributes.

Depending on the tasks, the normalization can be placed in a different position or even multiple layers~\cite{differentIN1, differentIN2} for better training. We find that in addition to offering a fine-grained control of features similar to computer vision tasks, SIN in DPAF plays a unique role in bounding $\ell_2$-sensitivity, though SIN is an easy modification of IN. Note that we apply SIN to only those feature maps just before DPAGG. Such a design is supported by our experiments that applying SIN in all the layers before DPAGG, in turn, degrades the utility because, unlike IN, SIN lacks the learnable parameters. 

\vspace{-0.2cm}\paragraph{Asymmetric Training of $D$ in DPAF.} 
In fact, AGG asks for a smaller batch size because, otherwise, the features will be mixed and cannot be recognized. Nevertheless, a smaller batch size, in turn, is harmful to DPSGD because the DP noise will make a greater impact on the gradient. Hence, DPAF prefers a larger batch size from the DPSGD point of view. We propose an asymmetric training strategy to resolve the contradicting requirements of setting a proper batch size. In essence, in the asymmetrical training, when training $D$, we update conv3* and FC* through SGD for every iteration but update conv2* through DPSGD($\varepsilon_2$) for every $\mu$ iterations, as shown in Figure~\ref{fig: The asymmetric training of DPAF}. $\mu$ is called \textit{asymmetry multiplier} because it determines the ratio of the privacy budget for conv2* and the budget for both conv3* and FC*. 

\begin{figure}[t]
\includegraphics[width=9cm]{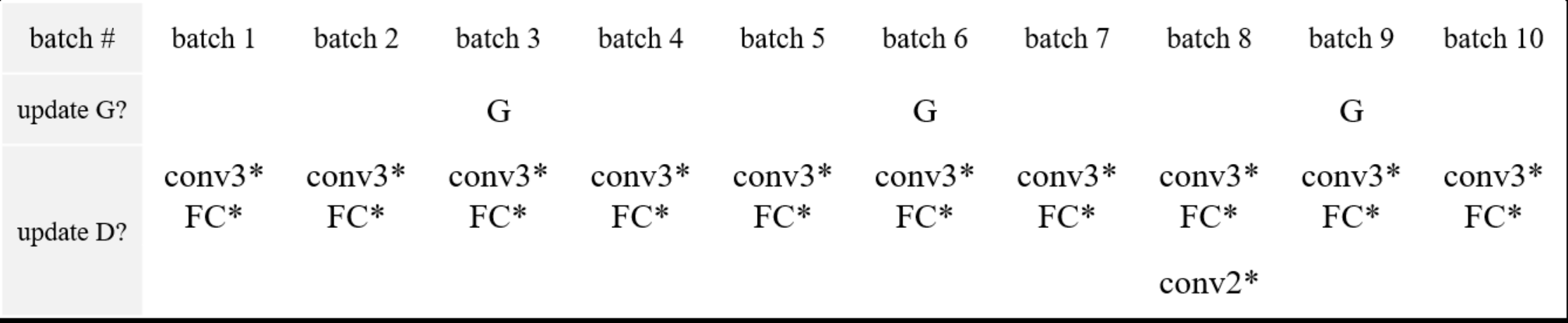}
\centering
\caption{The asymmetric training with $\mu=8$ and $n_{\text{critic}}=3$.}
\label{fig: The asymmetric training of DPAF}
\end{figure}

More specifically, given the $i$-th batch, $\mathcal{L}_{\text{BCE}}$ is calculated by feeding the samples from the $i$-th batch to $D$ and then we update conv3* and FC* through SGD. For the $i$-th batch with $\mu\mid i$, $\mathcal{L}'_{\text{BCE}}$ is calculated by feeding all of the samples from the $\mu$ latest batches to $\tilde{D}$ and then we update conv2* through DPSGD($\varepsilon_2$). Here, we make an important observation that updating conv2* through DPSGD($\varepsilon_2$) for every $\mu$ iterations virtually increases batch size $\mu$ times for conv2*. 

Given that $\mu$ controls the batch size for updating conv2*, a natural question that arises is whether $\mu$ can be increased arbitrarily. Unfortunately, we cannot arbitrarily increase $\mu$ because the increased $\mu$ also leads to a less frequent update of conv2*, which may in turn degrades the utility. 

\subsection{Discussion}\label{sec: Discussion}
Here, we have more discussions about the rationale behind the design of DPAF. 


\vspace{-0.2cm}\paragraph{Why Not Eliminate conv2*.} Consider the case where all of the layers before the DP feature aggregation belong to conv1*. The number of learnable parameters in $D$ will be much smaller (i.e., only conv3* and FC*); i.e., no conv2* exists. Such a setting hurts the training of GANs. This can be attributed to the fact that one knows from the GAN literature that if $G$ ($D$) is much stronger than $D$ ($G$), the training of GANs will likely fail to converge. In addition, conv3* and FC* may have fewer parameters compared to conv2*. It is difficult to well train $D$ under this circumstance. Thus, keeping certain layers as conv2* is beneficial for adversarial learning. 

\vspace{-0.2cm}\paragraph{Why Not More Layers for conv2*.} As more learnable parameters in $D$ may help the training of GANs, why conv2* does not have more layers? This can be explained as follows. 
If conv2* has more layers (parameters), because conv2* is updated through DPSGD, gradient clipping will lead to more information loss, flattening the feature values. In addition, if conv2* has more layers (parameters), because conv1 and conv2* both are trained by DPSGD, the output of (conv1, conv2*) will be too noisy, hindering the utility.

\vspace{-0.2cm}\paragraph{Why Not More Layers for conv3*.} A question that may arise is why conv3* does not have more layers. As the total number of layers is fixed given an input image, if conv3* has more layers, then either conv1 or conv2* (or both) will be shrunk. Thus, DPAGG is closer to low-level features. $D$ cannot have meaningful learning from the aggregation of level-level features. 

\vspace{-0.2cm}\paragraph{Choice of CGAN.} The DPAF is designed to support conditional generation. Thus, one needs to consider a CGAN in DPAF. Compared to GANs, $G$ and $D$ of CGANs need to consider the class label to ensure both the indistinguishability between the real and synthetic samples and the consistency between the input label and the label of synthetic samples. In general, there are two straightforward solutions for label injection. First, the class label is added as part of the input vector in such a case. If we feed labels to the input layer, the labels will be diluted in the forward phase and have a weak signal only. Second, $D$ is designed with two loss functions; one for the ordinary GAN loss and another for the class label. A representative of such a design is AC-GAN~\cite{acgan}, which outputs labels as part of the loss function. Nevertheless, due to access to the label, such a design leads to a privacy budget splitting and therefore suffers from utility degradation. 

In our design, DPAF follows the architecture of cDCGAN~\cite{cdcgan}. Inspired by~\cite{invertibleGAN} stating that the class label is better added to the first layer, we decide to use cDCGAN though there are no considerations of aggregation and DP in \cite{invertibleGAN}. In essence, cDCGAN feeds labels to the output of the input layer one by first computing the embedding (from scalar to vector) of labels, significantly strengthening the signal. A natural question that arises is why the class label is not added to the latter layers of $D$, given the class label in the latter layers may preserve an even stronger signal. The drawback of doing so is that all the layers before the layer to which the class label is added can hardly learn anything, because the classifier can know the portion by looking at the label only. 

\vspace{-0.2cm}\paragraph{The Position of DPAGG.}
DPAF heavily relies on DPAGG to raise the utility of synthetic samples. Thus, A natural question that arises is where the best position for DPAGG is. Without loss of generality, the DPAGG is placed to minimize the $\ell_2$-sensitivity $\Delta_{2, \text{AGG}}$ in Eq. (\ref{eq: AGG sensitivity}). As $\Delta_{2, \text{AGG}}=\sqrt{m}p$, $\Delta_{2, \text{AGG}}$ is dependent on the position of DPAGG. Given the design of a conventional CNN, where the number of neurons in the next convolutional layer is half of the one in the current convolutional layer, if the input is a $\rho \times \rho$ $c$-channel image, $\Delta_{2, \text{AGG}}$ can be calculated as $\sqrt{c\prod_{i=1}^a k_i}\cdot \frac{\rho}{2^a}$, where $k_i$ is the number of filters in the $i$-th convolutional layer and DPAGG is placed behind the $a$-th convolutional layer. Define $R_j$ as $\left(\sqrt{c\prod_{i=1}^j k_i}\cdot \frac{\rho}{2^j}\right) \big/ \left(\sqrt{c\prod_{i=1}^{j-1} k_i}\cdot \frac{\rho}{2^{j-1}}\right)$. We can easily derive $R_j=\sqrt{k_j}/2$. Thus, if $k_j\geq 4$, then $\Delta_{2, \text{AGG}}$ is monotone increasing from earlier to latter layers. From the $\ell_2$-sensitivity point of view, we conclude that the best position for DPAGG is between the first and second convolutional layers. Unfortunately, placing DPAGG in such a position does not lead to a decent utility in practice, because it completely destroys the structure of DPAF (e.g., the disappearance of conv1 and conv2*). Hence, we will empirically examine the other configurations in Section~\ref{sec: Experiment Evaluation}. 

\subsection{Privacy Analysis}
We rely on R\'{e}nyi DP (RDP)~\cite{RDP} for our privacy analysis. Compared to the ordinary DP in Definition~\ref{def: DP}, RDP is a variant of DP with a tighter bound of privacy loss.

\begin{definition} 
\label{def: RDP}
A randomized algorithm $\mathcal{M}$ is $(\alpha, \epsilon (\alpha))$-RDP with $\alpha>1$ if for any neighboring datasets $\mathcal{D}$ and $\mathcal{D}'$,
\begin{scriptsize}
\begin{align}
\label{eq: RDP}
D_\alpha(\mathcal{M}(\mathcal{D})||\mathcal{M}(\mathcal{D}'))=\frac{1}{\alpha-1}\log \mathbb{E}_{x \thicksim\mathcal{M}(\mathcal{D}')}\left[ \left( \frac{\text{Pr}[\mathcal(\mathcal{D})=x]}{\text{Pr}[\mathcal(\mathcal{D}')=x]} \right)^{\alpha-1}\right]\leq \epsilon (\alpha),
\end{align}
\end{scriptsize}
where $D_{\alpha}$ is the R\'{e}nyi divergence of order $\alpha$.
\end{definition}

Before proving our main result, we describe some necessary properties of RDP in Theorems~\ref{thm: gaussian mechanism in RDP}$\sim$\ref{thm: convert RDP to DP}. 



\begin{theorem}[Gaussian Mechanism on RDP~\cite{RDP, datalens}]
If a function $f$ has $\ell_{2}$-sensitivity $u$, then $G_{\sigma}\circ f$ obeys ($\alpha, \varepsilon(\alpha)$)-RDP, where $\varepsilon(\alpha)=\alpha u^{2} / (2\sigma^{2})$ and $G_{\sigma}$ is the Gaussian mechanism defined in Section~\ref{sec: Preliminary}.
\label{thm: gaussian mechanism in RDP}
\end{theorem}

\begin{theorem}[Sequential Composition on RDP~\cite{RDP}]
If the mechanism $\mathcal{M}_1$ satisfies $(\alpha, \epsilon_{1})$-RDP and the mechanism $\mathcal{M}_2$ satisfies $(\alpha, \epsilon_{2})$-RDP, then $\mathcal{M}_2\circ \mathcal{M}_1$ satisfies $(\alpha, \epsilon_{1} + \epsilon_{2})$-RDP.
\label{thm: gaussian composition in RDP}
\end{theorem}

\begin{theorem}[Privacy Amplification by Subsampling~\cite{Wang2019SubsampledRD}] Let $\mathcal{M} \circ$ {\rm\textsf{subsample}} be a randomized mechanism that first performs the subsampling without replacement with the subsampling rate $\gamma$ on the dataset $X$ and then takes as an input from the subsampled dataset $X^\gamma$. For all integers $\alpha \geq 2$, if $\mathcal{M}$ obeys $(\alpha, \epsilon (\alpha))$-RDP through Gaussian mechanism, then $\mathcal{M} \circ {\rm subsample}$ satisfies $(\alpha, \epsilon^{\prime}(\alpha))$-RDP, where \begin{footnotesize} \begin{equation} \begin{split} \epsilon^{\prime}(\alpha) \leq \frac{1}{(\alpha - 1)}\log (1 + \gamma^{2} \binom{\alpha}{2}\min \{ 4(e^{\epsilon (2)} - 1), e^{\epsilon (2)}\min \{2,\\ (e^{\epsilon (\infty)} - 1)^{2} \} \} + \sum_{j=3}^{\alpha} \gamma ^{j}\binom{\alpha}{j}e^{(j-1)\epsilon (j)}\min \{2, (e^{\epsilon (\infty)} - 1)^{j} \} ), \end{split} \end{equation} \end{footnotesize} and $\varepsilon(\alpha)=\alpha u^{2} / (2\sigma^{2})$ with $u$ as the sensitivity.\label{thm: subsampling amplification} \end{theorem}

In the following, we use $\epsilon^{\prime}(\alpha, \gamma, u)$ to indicate $\epsilon^{\prime}(\alpha)$ with the subsampling rate $\gamma$ and sensitivity $u$. 

\begin{theorem}[From RDP to DP~\cite{RDP}]
If a mechanism $\mathcal{M}$ is $(\alpha, \epsilon (\alpha))$-RDP, $\mathcal{M}$ is $(\epsilon (\alpha)+\frac{\log 1/\delta}{\alpha-1}, \delta)$-DP for any $\delta\in (0, 1)$.
\label{thm: convert RDP to DP}
\end{theorem}


\begin{theorem}
The proposed DPAF framework guarantees $(T_1\epsilon^{\prime} (\alpha, \gamma_1, u_1 ) + T_2\epsilon^{\prime} (\alpha, \gamma_2, u_2 ) + T_3\epsilon^{\prime} (\alpha, \gamma_3, \sqrt{m}p ) + \frac{\log \frac{1}{\delta}}{\alpha - 1}, \delta)$-DP for all $\alpha \geq 2$ and $\delta \in (0, 1)$.
\end{theorem}

\begin{proof}
For $C$, our goal is to ensure that the update of conv1 is satisfied with $(\alpha, \epsilon(\alpha))$-RDP for iteration. Note that because conv2, conv3, and FC will be discarded after the training, they do not need a DP guarantee. 
Let the total number of iterations for training $C$ be $T_1$. Then, the DPSGD (through the Gaussian mechanism in Theorem~\ref{thm: gaussian mechanism in RDP}) on conv1 is $(\alpha, T_1\epsilon (\alpha))$-RDP according to Theorem~\ref{thm: gaussian composition in RDP}. However, it can be re-estimated as $(\alpha , T_1 \epsilon^{\prime}(\alpha, \gamma_1, u_1))$-RDP with subsampling rate $\gamma _1$ according to Theorem~\ref{thm: subsampling amplification}. 



For $D$, since conv1's parameters are frozen during the training of $D$, the output of conv1 in $D$ is satisfied with $(\alpha , T_1 \epsilon^{\prime}(\alpha, \gamma_1, u_1))$-RDP, because the update of conv1 in $C$ has been proven to be RDP. 
Let the total number of iterations for training conv2* be $T_2$. Then, the DPSGD (through the Gaussian mechanism in Theorem~\ref{thm: gaussian mechanism in RDP}) on conv2* is $(\alpha, T_2\epsilon (\alpha))$-RDP according to and Theorem~\ref{thm: gaussian composition in RDP}. 
Similarly, the update of conv2* is satisfied with $(\alpha , T_2 \epsilon^{\prime}(\alpha, \gamma_2, u_2))$-RDP with subsampling rate $\gamma _2$ according to Theorem~\ref{thm: subsampling amplification}.
So far, the joint consideration of conv1 and conv2*, (conv1, conv2*), is satisfied with $(\alpha, T_1\epsilon ^{\prime} (\alpha, \gamma_1, u_1 ) + T_2\epsilon ^{\prime}(\alpha, \gamma_2, u_2))$-RDP guarantee according to Theorem~\ref{thm: gaussian composition in RDP}. 
Unlike the cases of conv1 and conv2*, where noise is injected in the backward phase, the noise injection to AGG occurs in the forward phase. More specifically, we set a DPAGG to aggregate input data and add noise in the aggregated data. 
The sensitivity of AGG has been calculated as $\sqrt{m}p$ in Eq. (\ref{eq: AGG sensitivity}). 
Let the number of iterations for the training of $D$ be $T_3$. 
The DPAGG with the noise sampled from $N(0, mp^2\sigma ^2 )$ is satisfied with $(\alpha, T_3\epsilon ^{\prime}(\alpha, \gamma_3, \sqrt{m}p))$-RDP with subsampling ratio $\gamma_3$ according to Theorem~\ref{thm: gaussian mechanism in RDP} and Theorem~\ref{thm: gaussian composition in RDP}. 
Thus, the joint consideration of conv1, conv2*, and DPAGG, (conv1, conv2*, DPAGG), will fullfil $(\alpha, T_1\epsilon ^{\prime}(\alpha, \gamma_1, u_1 ) + T_2\epsilon ^{\prime}(\alpha, \gamma_2, u_2) + T_3\epsilon ^{\prime}(\alpha, \gamma_3, \sqrt{m}p))$-RDP according to Theorem~\ref{thm: gaussian composition in RDP}. Because the DPAGG has DP guarantee, the update of conv3* and FC* is satisfied with RDP by the post-processing. Finally, the update of $G$ does not access the sensitive data and, as a result, is satisfied with RDP by the post-processing.
 

Overall, according to Theorem~\ref{thm: convert RDP to DP}, DPAF is satisfied with $(T_1\epsilon^{\prime} (\alpha, \gamma_1, u_1 ) + T_2\epsilon^{\prime} (\alpha, \gamma_2, u_2 ) + T_3\epsilon^{\prime} (\alpha, \gamma_3, \sqrt{m}p ) + \frac{\log \frac{1}{\delta}}{\alpha - 1}, \delta)$-DP.
\end{proof}

\section{Experiment Evaluation}\label{sec: Experiment Evaluation}
In this section, we present the experiment evaluation of DPAF. We chose representative datasets for the task of data synthesis. After that, DPAF was conducted to generate synthetic data. We then evaluated the utility of the synthetic data based on different settings of hyperparameters.

\subsection{Experiment Setup}\label{sec: Experiment Setup} 
We describe the datasets, baselines, evaluation metrics, and architecture of our canonical implementation of DPAF below. 

\vspace{-0.2cm}\paragraph{Dataset.} In our experiments, we considered MNIST, Fashion MNIST (FMNIST), CelebA, and FFHQ datasets. Both MNIST and FMNIST have $28\times 28$ grayscale images. CelebA contains colorful celebrity images of different sizes. In our experiments, we rescaled all of CelebA images into $64\times 64$ colorful images. Based on CelebA, we created two more datasets, CelebA-Gender and CelebA-Hair, where the former is for binary classification with gender as the label and the latter is for multiclass classification dataset with hair color (black/blonde/brown) as the label. FFHQ contains $1024\times 1024$ colorful facial images. We rescaled FFHQ images into $128\times 128$
images with the gender as the label\footnote{We used the FFHQ labels from \url{https://github.com/DCGM/ffhq-features-dataset/tree/master/json}.} and created FFHQ-Gender dataset for binary classification.

We followed the default training (60000 samples) and testing sets (10000 samples) for both MNIST and FMNIST. In addition, as CelebA does not have inherent training and testing sets, we followed \cite{7410782} to partition CelebA into training and testing sets. FFHQ contains 69471 images with 38388 females and 31083 males. We split the data into 90\% and 10\% for training and testing, respectively. 

\vspace{-0.2cm}\paragraph{Baselines.} We consider the baseline methods, GS-WGAN~\cite{gs-wgan}, DP-MERF~\cite{harder2021dpmerf}, P3GM~\cite{p3gm}, DataLens~\cite{datalens}, G-PATE~\cite{g-pate}, DP-Sinkhorn\cite{dp-sinkhorn}, and PEARL~\cite{pearl}. The implementation of all the baselines is based on the official code\footnote{The official code of GS-WGAN, DP-MERF, DataLens, G-PATE, and DP-Sinkhorn can be found at \url{https://github.com/DingfanChen/GS-WGAN}, \url{https://github.com/ParkLabML/DP-MERF}, \url{https://github.com/AI-secure/DataLens}, \url{https://github.com/AI-secure/G-PATE}, and \url{https://github.com/nv-tlabs/DP-Sinkhorn_code}, respectively}. Though the official code is not available online, we communicated with the authors of PEARL to have a copy. 

Most of the official codes are written for synthesizing images of low resolutions. As we mainly focus on the synthesis of high-dimensional images, we made necessary modifications such as batch size and input size to make them adaptable to different settings.  

\vspace{-0.2cm}\paragraph{Evaluation Metrics.} Given two levels of privacy guarantee, $(1, 10^{-5})$-DP and $(10, 10^{-5})$-DP,  we aim to evaluate the utility of DP image synthesis. The utility can have two dimensions; i.e., the classification accuracy and the visual quality. In the former case, we calculate the predicting accuracy of the classifier trained by synthetic images and tested by real images. The architecture of the classifier used in our experiment is the same as the one used in DataLens~\cite{datalens}. On the other hand, in the latter case, we display the synthetic images for visualization and report Fr\'{e}chet inception distance (FID). 

\begin{table*}[hbt!]
    \centering
    \small
    \addtolength{\tabcolsep}{-2pt}    
    \begin{tabular}{@{}|c|l|ccccccc||c|@{}}
    \hline & \hspace{2ex}$\varepsilon$   & \shortstack{GS-WGAN} & \shortstack{DP-MERF} & \shortstack{P3GM} & \shortstack{G-PATE} & \shortstack{DP-Sinkhorn} & \shortstack{DataLens} & \shortstack{PEARL} &\shortstack{\textsf{DPAF} } \\ 
    \hline\hline
    \multirow{2}{*}{MNIST}         & $\varepsilon = 1$& 0.143  & 0.637  & 0.737 & 0.588 & 0.654 & 0.712 & \textbf{0.782} & 0.501 \\
    & $\varepsilon = 10$ & 0.808 & 0.674     &0.798 & 0.809 & \textbf{0.832} & 0.807 & 0.788 & 0.748      \\ 
 
    \hline
    \multirow{2}{*}{\shortstack{Fashion-MNIST}} & $\varepsilon = 1$ & 0.166  & 0.586 & \textbf{0.722} & 0.581 & 0.564 & 0.648 & 0.683 & 0.543      \\
    & $\varepsilon = 10$ & 0.658  & 0.616   & \textbf{0.748} & 0.693  & 0.711 & 0.706 & 0.649 & 0.640     \\
    \hline
    \multirow{2}{*}{\shortstack{CelebA-Gender}} & $\varepsilon = 1$ & 0.590 & 0.594  & 0.567 & 0.670  & 0.543 & 0.700 & 0.634 & \textbf{0.802}    \\
    & $\varepsilon = 10$ & 0.614     & 0.608    &  0.588  & 0.690 & 0.621 & 0.729 & 0.646 &\textbf{0.826}  \\ 

    \hline
    \multirow{2}{*}{\shortstack{CelebA-Hair}}  & $\varepsilon = 1$ & 0.420  & 0.441   &  0.453  & 0.499 & $\times$ & 0.606 & 0.606 & \textbf{0.675}      \\
    & $\varepsilon = 10$ & 0.523 & 0.449   &  0.486  & 0.622 & $\times$ & 0.622 & 0.626 &\textbf{0.671}      \\ 
    \hline
    \end{tabular}
    \vspace{-0.3cm}
    \caption{Classification accuracy results under $(1, 10^{-5})$-DP and $(10, 10^{-5})$-DP. The rightmost column shows canonical accuracy.}
    \label{tab: classification}
\end{table*}

\vspace{-0.2cm}\paragraph{Canonical Implementation of DPAF.} Basically, DPAF adds the DP feature aggregation on the basis of cDCGAN. In our canonical implementation, the batch size is $24$ for MNIST and FMNIST and $64$ for CelebA and FFHQ. The latent vector sampled from the standard Gaussian distribution is of dimension 100. The asymmetry multiplier $\mu=8$. We also apply gradient compression~\cite{gradientcompression} to the per-sample gradient to keep the top 90\% values only. Our DPAF is configured to be C2-C1-$\times$ for MNIST/FMNIST, C2-C2-C1 for CelebA, and C3-C1-$\times$ for FFHQ, where the notation C$x_1$-C$x_2$-C$x_3$ means that the $D$ uses $x_1$ layers as conv1, $x_2$ layers as conv2*, and $x_3$ layers as conv3*. The notation $\times$ means that the corresponding layer does not exist. We always have two FC* layers.

The notation $(x_1\%, x_2\%, x_3\%)$ refers to the setting, where conv1, conv2*, and DPAGG have $\varepsilon_1=\frac{x_1\cdot \varepsilon}{100}$, $\varepsilon_2=\frac{x_2\cdot \varepsilon}{100}$, and $\varepsilon_3=\frac{x_3\cdot \varepsilon}{100}$, respectively, given the total privacy budget $\varepsilon$. A similar notation is $(x_1, \times, x_3)$, where both conv1 and DPAGG have a privacy budget $\varepsilon_1=x_1$ and $\varepsilon_3=x_3$, respectively, and conv2* has the rest. For example, $(0.1, \times, 0.1)$ means that conv1, conv2, and DPAGG have $0.1$, $9.8$, and $0.1$, respectively, if the total privacy budget is $10$. Throughout Section~\ref{sec: Experiment Results}, \textit{canonical accuracy} means the accuracy from the canonical implementation. 

\subsection{Experiment Results}\label{sec: Experiment Results}
Here, we present our experiment results and ablation study. All of the experiment results below are derived by averaging the results from five independent experiments. 

\subsubsection{Classification Accuracy}\label{sec: Classification Accuracy} Table~\ref{tab: classification} shows the classification results of DPAF and the other baseline methods. One can see from Table~\ref{tab: classification} that DPAF outperforms all of the other baseline methods for CelebA-Gender and CelebA-Hair but, surprisingly, is worse than some of the other baseline methods for MNIST and FMNIST. This can be attributed to the architecture behind the design of DPAF. Particularly, as mentioned in Sections~\ref{sce: Training a Classifier Before Transfer Learning} and \ref{sec: Training a DPGAN After Transfer Learning}, our convolutional layers follow the conventional design; i.e., images are half-sized from the current layer to the next layer, and consequently, the number of convolutional layers is dependent on the input image size. For MNIST and FMNIST, as the image size is smaller, there will be fewer convolutional layers (e.g., C2-C1-$\times$ in our canonical design), thus restricting the generative capability of the generator, given the conventional design that $G$ is the reverse of $D$ (e.g., PGGAN~\cite{pggan}). On the contrary, the larger image size implies more convolutional layers in DPAF, strengthening the generative capability. In this sense, our design also suggests the potential of DPAF in synthesizing images of higher resolutions because a generator with more layers can be adopted. Such an argument is partially supported by the experiment results in Section~\ref{sec: High Resolution 128 Image Synthesis}.

\begin{figure}[htb!]
\centering
\includegraphics[width=0.48\textwidth]{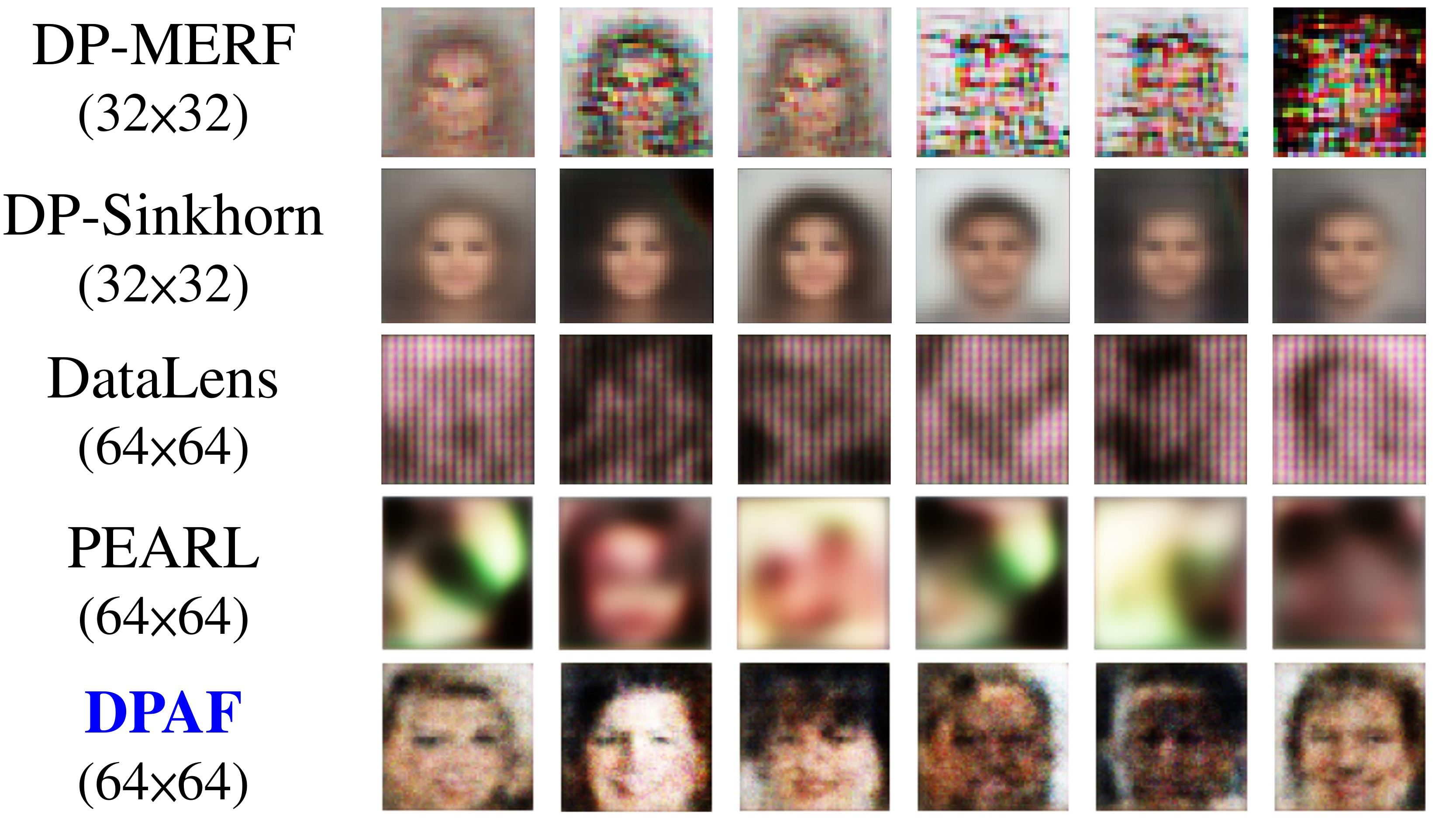}
\caption{Visual results for CelebA-Gender with $\varepsilon = 10$. The left (right) three columns are females (males).}
\label{fig: visual-result-celeba}
\end{figure}

\begin{table}[hbt!]
    \centering
    \small
    \begin{tabular}{@{}|c|l|ccc|@{}}
    \hline & \hspace{2ex}$\varepsilon$   & \shortstack{DataLens}  & \shortstack{PEARL} &\shortstack{\textsf{DPAF} } \\
    \hline\hline
    \multirow{2}{*}{\shortstack{CelebA-Gender}} & $\varepsilon = 1$ & 298  & 303 & \textbf{285} \\ 
    & $\varepsilon = 10$ & 320 & 302 & \textbf{298} \\ 

    \hline
    \multirow{2}{*}{\shortstack{CelebA-Hair}}  & $\varepsilon = 1$ & $\times$ & 338  & \textbf{301} \\ 
    & $\varepsilon = 10$ & $\times$  & 337 & \textbf{298} \\ 
    \hline
    
    \end{tabular}

    \vspace{-0.3cm}
    \caption{The comparison of FIDs for CelebA.}
    \label{tab: fid for celeba}
\end{table}

\subsubsection{Visual Quality}
We first present the visual quality evaluation results in Figure~\ref{fig: visual-result-celeba}. In particular, the DPAF-synthesized images appear more realistic and capture more facial features, such as eyes, lips, and face shape, compared to the uniformed faces generated by DP-Sinkhorn and the highly noisy faces generated by the other baselines. Such a gain comes from the use of SIN and our design of DPAGG; i.e., the aggregated feature is more robust to the DP noise and better to keep the features distinguished after the training. We also present the quantitative results in Table~\ref{tab: fid for celeba}. The lowest FIDs of DPAF in different settings among the baselines are consistent with the visual quality comparison in Figure~\ref{fig: visual-result-celeba}.

\begin{figure}[htb]
\centering
\includegraphics[width=0.48\textwidth]{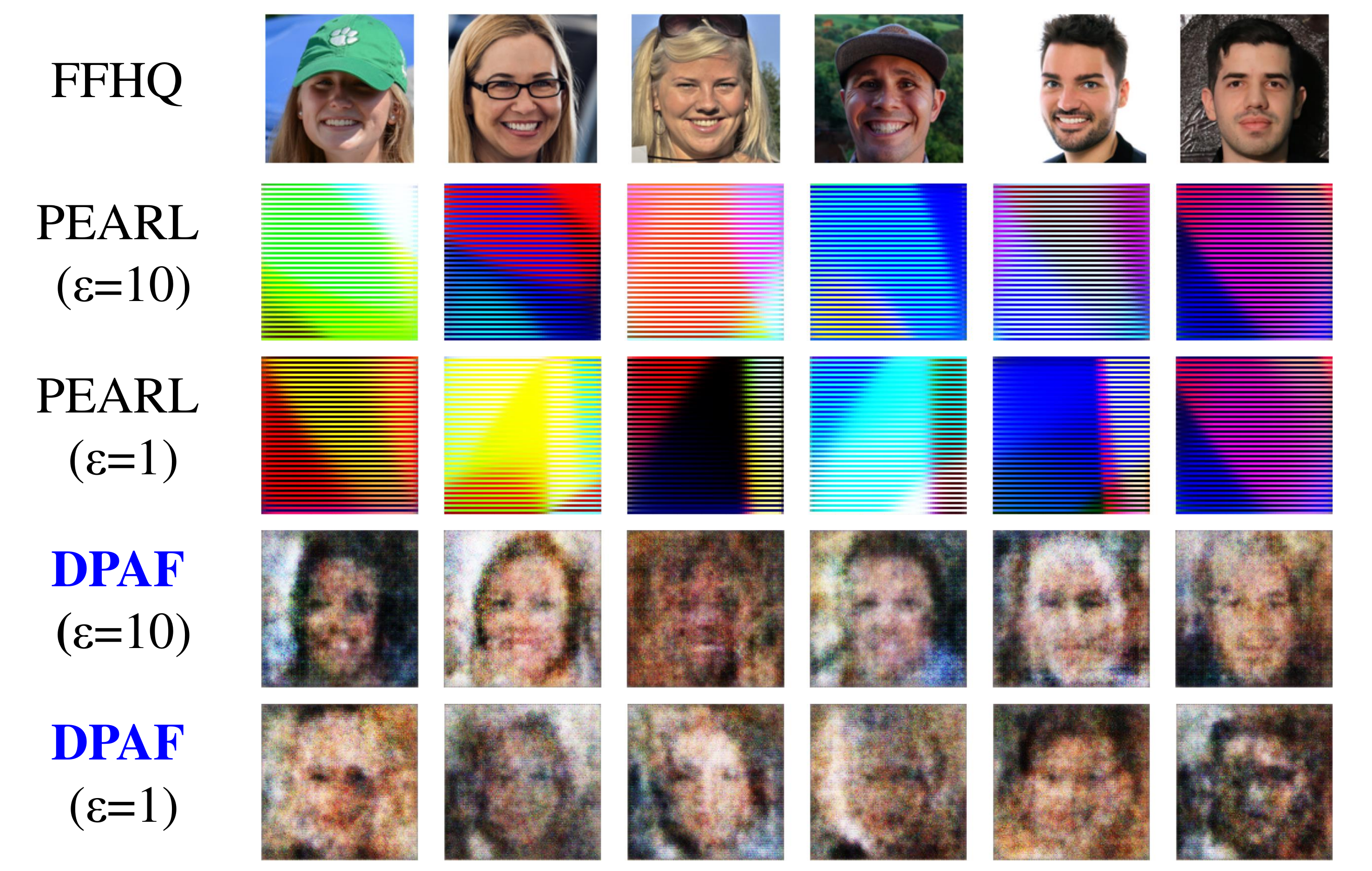}
\caption{Visual results for FFHQ. The left (right) three columns are females (males).}
\label{fig: visual-result-FFHQ}
\end{figure}

\begin{table}[hbt!]
    \centering
    \small
    \begin{tabular}{@{}|c|l|cc|@{}}
    \hline & \hspace{2ex}$\varepsilon$   &\shortstack{PEARL}  & \shortstack{DPAF}  \\
    \hline\hline
    \multirow{2}{*}{\shortstack{FFHQ-Gender}} & $\varepsilon = 1$ & 0.441 $\pm$ 0.019   & \textbf{0.567} $\pm$ 0.038 \\
    & $\varepsilon = 10$ & 0.511 $\pm$ 0.027 & \textbf{0.646} $\pm$ 0.004 \\ 

    \hline
    
    \end{tabular}

    \vspace{-0.3cm}
    \caption{The comparison of accuracy for FFHQ-Gender.}
    \label{tab: ablation_8}
\end{table}

\subsubsection{High Resolution $128\times 128$ Image Synthesis}\label{sec: High Resolution 128 Image Synthesis}
To further demonstrate the advantage of DPAF in synthesizing high-utility and high-dimensional images, we conducted experiments on FFHQ-Gender and the visual results are shown in Figure~\ref{fig: visual-result-FFHQ}, where the images synthesized by PEARL look like a pure noise while the images synthesized by DPAF still preserve facial features. Table~\ref{tab: ablation_8} reports the predicting accuracy on FFHQ-Gender. In Section~\ref{sec: Classification Accuracy}, we claim that the generative capability of DPAF will be increased with the image size. At first sight, we can see from Table~\ref{tab: ablation_8} that the accuracy goes worse compared to CelebA-Gender in Table~\ref{tab: classification}. Nonetheless, this can be explained as follows. First, CelebA-Gender and FFHQ-Gender are two different datasets, having different distributions. The direct comparison between the accuracy of CelebA-Gender and FFHQ-Gender remains doubtful. Second, synthesizing $128\times 128$ images has reached the limit of the conventional GANs without using modern techniques such as residual blocks~\cite{7780459}. Synthesizing images of higher resolutions requires much more sophisticated GANs (e.g., PGGAN~\cite{pggan}. Applying the techniques in DPAF to modern GANs remains to be unexplored but would be our research direction in the future. 

\subsubsection{Privacy Budget Allocation w/o Transfer Learning}\label{sec: Different Strategies for Privacy Budget Allocation without Transfer Learning}
Because we allocate a very limited privacy budget to conv1 in our canonical implementation, a natural question that arises is whether transfer learning is necessary. In other words, if conv1 cannot effectively learn from the data, a reasonable design choice is to abandon transfer learning and invest the privacy budget in DPAGG. Below we examine the predicting accuracy under three strategies for privacy budget allocation in the absence of transfer learning. 

\begin{table}[hbt!]
    \centering
    \scriptsize
    \begin{tabular}{@{}|c|l|ccc|@{}}
    \hline & \hspace{2ex}$\varepsilon$   &\shortstack{DPAF}  & \shortstack{($\times$, 0.5)} & \shortstack{($\times$, 0.2)} \\ 
    \hline\hline
    \multirow{2}{*}{\shortstack{CelebA-Gender}} & $\varepsilon = 1$ & \textbf{0.802} $\pm$ 0.018   & 0.752 $\pm$ 0.045 & 0.774 $\pm$ 0.012  \\
    & $\varepsilon = 10$ & \textbf{0.826} $\pm$ 0.010 & 0.793 $\pm$ 0.024 & 0.700 $\pm$ 0.093  \\ 

    \hline
    \multirow{2}{*}{\shortstack{CelebA-Hair}}  & $\varepsilon = 1$ & \textbf{0.675} $\pm$ 0.013 & 0.635 $\pm$ 0.035  & 0.667 $\pm$ 0.029   \\
    & $\varepsilon = 10$ & 0.671 $\pm$ 0.014  & \textbf{0.681} $\pm$ 0.015 & 0.670 $\pm$ 0.016  \\ 
    \hline
    \end{tabular}
    \vspace{-0.3cm}
    \caption{The accuracy of random parameters for conv1.}
    \label{tab: ablation_1}
\end{table}

\vspace{-0.2cm}\paragraph{Random Parameters for conv1.} We consider random weights of conv1; i.e., all of the weights in conv1 are sampled uniformly at random from the zero-mean Gaussian distribution with a standard deviation of $0.02$ and will never be updated. The results are shown in Table~\ref{tab: ablation_1}, where the notations $(\times, x_3)$ means that privacy budget $\varepsilon_3=x_3$ is allocated to DPAGG while the rest of budget is for conv2*. One can see that the accuracy from ($\times$, 0.5) and ($\times$, 0.2) is only slightly lower than the canonical accuracy. As conv1 is supposed to learn low-level features, even if conv1 uses random features, the learning of conv2*, conv3*, and FC* can be adapted to random conv1 and be performed well. However, according to our empirical experience, we still spend a very limited privacy budget on conv1 to avoid undesirable cases, where some feature maps happen to contain only zero's or nearly zero values, making such feature maps useless. The higher variances of ($\times$, 0.5) and ($\times$, 0.2) also justify the above design choice.

\begin{table}[hbt!]
    \centering
    \scriptsize
    \begin{tabular}{@{}|c|l|ccc|@{}}
    \hline & \hspace{2ex}$\varepsilon$   &\shortstack{DPAF}  & \shortstack{(0.1, $\times$ , 0.1)} & \shortstack{(33\%, 34\%, 33\%)} \\ 
    \hline\hline
    \multirow{2}{*}{\shortstack{CelebA-Gender}} & $\varepsilon = 1$ & \textbf{0.802} $\pm$ 0.018   & 0.594 $\pm$ 0.112 & 0.727 $\pm$ 0.143  \\
    & $\varepsilon = 10$ & \textbf{0.826} $\pm$ 0.010 & 0.747 $\pm$ 0.100 & 0.817 $\pm$ 0.024  \\ 

    \hline
    \multirow{2}{*}{\shortstack{CelebA-Hair}}  & $\varepsilon = 1$ & \textbf{0.675} $\pm$ 0.013 & 0.424 $\pm$ 0.079  & 0.642 $\pm$ 0.038   \\
    & $\varepsilon = 10$ & 0.671 $\pm$ 0.014  & 0.599 $\pm$ 0.116 & \textbf{0.685} $\pm$ 0.018   \\ 
    \hline
    \end{tabular}
    \vspace{-0.3cm}
    \caption{The classification accuracy of updating conv1 during training of $D$.}
    \label{tab: ablation_2}
\end{table}

\vspace{-0.2cm}\paragraph{Updating conv1 During Training of $D$.} Here, we do not perform transfer learning but still use DPSGD($\varepsilon_1$) to update conv1 during the training of $D$. Table~\ref{tab: ablation_2} shows the results, where the canonical implementation outperforms the other configurations. There are two reasons.  First, while the update of conv1 is accomplished during the training of $C$, because conv2, conv3, and FC in $C$ are updated by SGD, conv1 is more informative. While the update of conv1 is accomplished during the training of $D$, conv1 is less informative because of the noise accumulation. Second, in the absence of transfer learning, more layers (parameters) need to be updated during the training of $D$, which is more difficult to train them well from the adversarial learning perspective. Note that $(0.1, \times, 0.1)$ in the second column of Table~\ref{tab: ablation_2} means that we update conv1 through DPSGD($\varepsilon_1$) with $\varepsilon_1=0.1$ in the training of $D$, though the same $(0.1, \times, 0.1)$ in the context of using transfer learning has a different interpretation, as shown in Section~\ref{sec: Experiment Setup}.


\begin{table}[hbt!]
    \centering
    \scriptsize
    \begin{tabular}{@{}|c|l|ccc|@{}}
    \hline
     & \hspace{2ex}$\varepsilon$   &\shortstack{DPAF}  & \shortstack{(50\%, 50\%)} & \shortstack{($\times$, 0.1)} \\ 
    \hline\hline
    \multirow{2}{*}{\shortstack{CelebA-Gender}} & $\varepsilon = 1$ & \textbf{0.802} $\pm$ 0.018   & 0.768 $\pm$ 0.036 &  0.748 $\pm$ 0.072  \\
    & $\varepsilon = 10$ & \textbf{0.826} $\pm$ 0.010 & 0.759 $\pm$ 0.030 &   0.790 $\pm$ 0.020\\ 

    \hline
    \multirow{2}{*}{\shortstack{CelebA-Hair}}  & $\varepsilon = 1$ & \textbf{0.675} $\pm$ 0.013 & 0.661 $\pm$ 0.034  &   0.643 $\pm$ 0.028 \\
    & $\varepsilon = 10$ & \textbf{0.671} $\pm$ 0.014  & 0.648 $\pm$ 0.038 &   0.660 $\pm$ 0.018 \\ 
    \hline
    
    \end{tabular}

    \vspace{-0.3cm}
    \caption{The classification accuracy of joint updating conv1 and conv2* during training of $D$.}
    \label{tab: ablation_3}
\end{table}

\vspace{-0.2cm}\paragraph{Joint Updating conv1 and conv2* During Training of $D$.}
In this case, conv1 and conv2* are jointly considered and will be updated through the DPSGD together. We see (conv1, conv2*) as a larger component. Compared to the individual conv1 and conv2*, gradient clipping in DPSGD may cause more severe information loss and DPSGD will cause more severe damage of the gradient structure. The results in Table~\ref{tab: ablation_3} support the above arguments. 

Furthermore, by comparing the $(0.1, \times, 0.1)$ column in Table~\ref{tab: ablation_2} and $(\times, 0.1)$ column in Table~\ref{tab: ablation_3} (due to their similar setting), we can find that the accuracy in the former is consistently lower than the one in the latter. Unlike Table~\ref{tab: ablation_3}, conv1 and con2* in Table~\ref{tab: ablation_2} are treated separately and hence suffer from privacy budget splitting, resulting in worse accuracy. 

\begin{table*}[hbt!]
    \centering
    \small
    \addtolength{\tabcolsep}{-4pt}    
    \begin{tabular}{@{}|c|l|cccccc|@{}}
    \hline & \hspace{2ex}$\varepsilon$   &\shortstack{DPAF $(0.1, \times, 0.1)$}  & \shortstack{(20\%, 20\%, 60\%)} & \shortstack{(20\%, 60\%, 20\%)} & \shortstack{(20\%, 40\%, 40\%)} & \shortstack{(30\%, 20\%, 50\%)} & \shortstack{(30\%, 50\%, 20\%)} \\ 
    \hline\hline
    \multirow{2}{*}{\shortstack{CelebA-Gender}} & $\varepsilon = 1$ & \textbf{0.802} $\pm$ 0.018   & 0.741 $\pm$ 0.074 & 0.801 $\pm$ 0.035  & 0.793 $\pm$ 0.030 & 0.771 $\pm$ 0.035  & 0.795 $\pm$ 0.033 \\
    & $\varepsilon = 10$ & \textbf{0.826} $\pm$ 0.010 & 0.787 $\pm$ 0.018 & 0.820 $\pm$ 0.020  &  0.790 $\pm$ 0.017 & 0.767 $\pm$ 0.040 & 0.813 $\pm$ 0.013 \\ 

    \hline
    \multirow{2}{*}{\shortstack{CelebA-Hair}}  & $\varepsilon = 1$ & 0.675 $\pm$ 0.013 & 0.570 $\pm$ 0.144  & 0.663 $\pm$ 0.031   &  0.657 $\pm$ 0.023  & 0.643 $\pm$ 0.046 & 0.666 $\pm$ 0.024 \\
    & $\varepsilon = 10$ & \textbf{0.671} $\pm$ 0.014  & 0.639 $\pm$ 0.031 & 0.619 $\pm$ 0.038   &  0.659 $\pm$ 0.016  & 0.598 $\pm$ 0.050 & 0.577 $\pm$ 0.094      \\ 
    \hline
    
    \end{tabular}

    \vspace*{3 pt}
    
    \begin{tabular}{@{}|c|l|ccccc|@{}}
    \hline & \hspace{2ex}$\varepsilon$   & \shortstack{(30\%, 40\%, 30\%)} & \shortstack{(30\%, 30\%, 40\%)} & \shortstack{(40\%, 30\%, 30\%)} & \shortstack{(40\%, 20\%, 40\%)} & \shortstack{(40\%, 40\%, 20\%)} \\ 
    \hline\hline
    \multirow{2}{*}{\shortstack{CelebA-Gender}} & $\varepsilon = 1$ & 0.791 $\pm$ 0.012 & 0.704 $\pm$ 0.164  & 0.745 $\pm$ 0.102 & 0.759 $\pm$ 0.025  & 0.763 $\pm$ 0.038 \\
    & $\varepsilon = 10$ & 0.799 $\pm$ 0.022 & 0.799 $\pm$ 0.023  &  0.797 $\pm$ 0.015 & 0.782 $\pm$ 0.032 & 0.798 $\pm$ 0.022 \\ 

    \hline
    \multirow{2}{*}{\shortstack{CelebA-Hair}}  & $\varepsilon = 1$ & \textbf{0.677} $\pm$ 0.011  & 0.653 $\pm$ 0.024   &  0.646 $\pm$ 0.039  & 0.454 $\pm$ 0.150 & 0.641 $\pm$ 0.064 \\
    & $\varepsilon = 10$ & 0.610 $\pm$ 0.018 & 0.643 $\pm$ 0.012   &  0.628 $\pm$ 0.011  & 0.634 $\pm$ 0.036 & 0.645 $\pm$ 0.021     \\ 
    \hline
    
    \end{tabular}

    \vspace{-0.3cm}
    \caption{The classification accuracy of different privacy budget allocations with transfer learning.}
    \label{tab: ablation_0}
\end{table*}

\subsubsection{Privacy Budget Allocation w/ Transfer Learning}\label{sec: Different Strategies for Privacy Budget Allocation for Transfer Learning} 
Given the regular use of DPAF (i.e., DPAF with transfer learning), we aim to examine the impact of different budget allocations on accuracy. 

A general guideline for allocating privacy budgets is that the earlier (latter) layers should earn more (fewer) budgets. The rationale is that the earlier layers learn the low-level features and the latter layers will be adapted to the low-level features. Once the earlier layers have only a limited budget and the parameters fluctuate, the latter layers can hardly be adapted to the fast change of earlier layers and can hardly learn informative parameters. However, during the training of $D$, conv1 is frozen and does not need to be updated. In addition, as mentioned in Section~\ref{sec: Different Strategies for Privacy Budget Allocation without Transfer Learning}, $\varepsilon_1$ can be a small value. The gradient vector may have many small or even nearly zero values, which can easily be affected by the DP noise. On the other hand, the aggregated vector output by DPAGG is designed to have larger values, as described in Section~\ref{sec: Overview}, for better robustness against the DP noise. Thus, a reasonable choice is to $\varepsilon_2>\varepsilon_3$. 

The above arguments can be confirmed empirically because we can see from Table~\ref{tab: ablation_0} that the canonical setting $(0.1, \times, 0.1)$ that follows the above discussion outperforms the other settings. 

\begin{table*}[hbt!]
    \centering
    \small
    \vspace{-7pt}
    \addtolength{\tabcolsep}{-2pt}    
    \begin{tabular}{@{}|c|l|cccccc|@{}}
    \hline & \hspace{2ex}$\varepsilon$   & \shortstack{C1-C1-$\times$}  & \shortstack{C1-C2-$\times$} & \shortstack{C2-C1-$\times$} & \shortstack{C1-C3-$\times$} & \shortstack{C2-C2-$\times$} & \shortstack{C3-C1-$\times$} \\
    \hline\hline
    \multirow{2}{*}{\shortstack{CelebA-Gender}} & $\varepsilon = 1$ & 0.629 $\pm$ 0.040 & 0.737 $\pm$ 0.025  & \textbf{0.824} $\pm$ 0.025 & 0.661 $\pm$ 0.144  & 0.805 $\pm$ 0.021  & 0.811 $\pm$ 0.020 \\
    & $\varepsilon = 10$ & 0.720 $\pm$ 0.045 & 0.733 $\pm$ 0.038  &  0.762 $\pm$ 0.079 & 0.729 $\pm$ 0.032 & \textbf{0.786} $\pm$ 0.018  & 0.751 $\pm$ 0.036 \\ 

    \hline
    \multirow{2}{*}{\shortstack{CelebA-Hair}}  & $\varepsilon = 1$ & 0.423 $\pm$ 0.089  & 0.475 $\pm$ 0.120   &  0.643 $\pm$ 0.014  & \textbf{0.662} $\pm$ 0.019 & 0.519 $\pm$ 0.095  & 0.623 $\pm$ 0.038\\
    & $\varepsilon = 10$ & 0.565 $\pm$ 0.019 & 0.644 $\pm$ 0.024   &  0.639 $\pm$ 0.039  & 0.657 $\pm$ 0.023 & \textbf{0.683} $\pm$ 0.022  & 0.659 $\pm$ 0.019  \\ 
    \hline
    \end{tabular}
    \vspace{-0.3cm}
    \caption{The classification accuracy of different layer architecture for conv1 and  conv2*.}
    \label{tab: ablation_4}
\end{table*}

\begin{table*}[hbt!]
    \centering
    \small
    \vspace{-8pt}
    \begin{tabular}{@{}|c|l|ccccc|@{}}
    \hline & \hspace{2ex}$\varepsilon$   &\shortstack{DPAF (C2-C2-C1)}  & \shortstack{C1-C3-C1} & \shortstack{C2-C1-C1} & \shortstack{C2-C1-C2} & \shortstack{C3-C1-C1} \\
    \hline\hline
    \multirow{2}{*}{\shortstack{CelebA-Gender}} & $\varepsilon = 1$ & \textbf{0.802} $\pm$ 0.018   & 0.448 $\pm$ 0.164 & 0.673 $\pm$ 0.099  & 0.505 $\pm$ 0.067 & 0.800 $\pm$ 0.017 \\ & $\varepsilon = 10$ & \textbf{0.826} $\pm$ 0.010 & 0.727 $\pm$ 0.039 & 0.803 $\pm$ 0.022  &  0.514 $\pm$ 0.091 & 0.820 $\pm$ 0.015 \\ 

    \hline
    \multirow{2}{*}{\shortstack{CelebA-Hair}}  & $\varepsilon = 1$ & \textbf{0.675} $\pm$ 0.013 & 0.356 $\pm$ 0.077  & 0.540 $\pm$ 0.036   &  0.354 $\pm$ 0.038  & 0.669 $\pm$ 0.018 \\ & $\varepsilon = 10$ & \textbf{0.671} $\pm$ 0.014  & 0.368 $\pm$ 0.074 & 0.664 $\pm$ 0.035   &  0.352 $\pm$ 0.055  & 0.670 $\pm$ 0.016 \\ 
    \hline
    
    \end{tabular}

    \vspace{-0.3cm}
    \caption{The classification accuracy of different layer architecture for conv1,  conv2*, and conv3*.}
    \label{tab: ablation_5}
\end{table*}

\begin{table*}[hbt!]
    \centering
    \small
    \vspace{-8pt}
    \begin{tabular}{@{}|c|l|ccccc|@{}}
    \hline & \hspace{2ex}$\varepsilon$   &\shortstack{$\mu=2$}  & \shortstack{$\mu=4$} & \shortstack{DPAF ($\mu=8$)} & \shortstack{$\mu=10$} & \shortstack{$\mu=20$}\\
    \hline\hline
    \multirow{2}{*}{\shortstack{CelebA-Gender}} & $\varepsilon = 1$ & 0.775 $\pm$ 0.037   & 0.688 $\pm$ 0.070 & \textbf{0.802} $\pm$ 0.018 & 0.772 $\pm$ 0.036 & 0.773 $\pm$ 0.030 \\ & $\varepsilon = 10$ & 0.665 $\pm$ 0.162 & 0.819 $\pm$ 0.033 & \textbf{0.826} $\pm$ 0.010 & 0.745 $\pm$ 0.010 & 0.735 $\pm$ 0.110\\ 

    \hline
    \multirow{2}{*}{\shortstack{CelebA-Hair}}  & $\varepsilon = 1$ & 0.578 $\pm$ 0.094 & 0.642 $\pm$ 0.038  & \textbf{0.675} $\pm$ 0.013 & 0.656 $\pm$ 0.017 & 0.648 $\pm$ 0.011 \\ & $\varepsilon = 10$ & 0.670 $\pm$ 0.027  & 0.666 $\pm$ 0.041 & \textbf{0.671} $\pm$ 0.014 & 0.670 $\pm$ 0.015 & 0.668 $\pm$ 0.020 \\ 
    \hline
    
    \end{tabular}

    \vspace{-0.3cm}
    \caption{The classification accuracy of different asymmetry multipliers $\mu$'s.}
    \label{tab: ablation_6}
\end{table*}

\begin{table}[hbt!]
    \centering
    \small
    \vspace{-15pt}
    \addtolength{\tabcolsep}{-3pt}    
    \begin{tabular}{@{}|c|l|ccc|c|c|@{}|}
    \hline & \hspace{2ex}$\varepsilon$ &\shortstack{DPAF}   &\shortstack{Trans($D$)}  &\shortstack{Trans($G+D$)}  \\
    \hline\hline
    \multirow{2}{*}{\shortstack{CelebA\\-Gender}} & $\varepsilon = 1$ & 0.802 $\pm$ 0.018 & 0.819 $\pm$ 0.020 & \textbf{0.820} $\pm$ 0.027 \\
    & $\varepsilon = 10$ & 0.826 $\pm$ 0.010 & 0.830 $\pm$ 0.018 & \textbf{0.831} $\pm$ 0.023  \\ 

    \hline
    \multirow{2}{*}{\shortstack{CelebA\\-Hair}}  & $\varepsilon = 1$ & 0.675 $\pm$ 0.013 & 0.663 $\pm$ 0.038 & \textbf{0.695} $\pm$ 0.015   \\
    & $\varepsilon = 10$ & 0.671 $\pm$ 0.014 & \textbf{0.695} $\pm$ 0.036 & 0.693 $\pm$ 0.012  \\ 
    \hline
    \end{tabular}
    \vspace{-0.3cm}
    \caption{The accuracy of DPAF with public data pre-training.}
    \label{tab: ablation_10}
\end{table}

\begin{table}[hbt!]
    \centering
    \small
    \vspace{-8pt}
    \addtolength{\tabcolsep}{-3pt}    
    \begin{tabular}{@{}|c|l|ccc|@{}}
    \hline & \hspace{2ex}$\varepsilon$ & \shortstack{DPAF}  &\shortstack{w/o GC}  & \shortstack{TOPAGG~\cite{datalens}} \\
    \hline\hline
    \multirow{2}{*}{\shortstack{CelebA\\-Gender}} & $\varepsilon = 1$ & \textbf{0.802} $\pm$ 0.018 & 0.725 $\pm$ 0.150   & 0.549 $\pm$ 0.075\\ 
    & $\varepsilon = 10$ & \textbf{0.826} $\pm$ 0.010 & 0.818 $\pm$ 0.022  & 0.530 $\pm$ 0.097 \\ 

    \hline
    \multirow{2}{*}{\shortstack{CelebA\\-Hair}}  & $\varepsilon = 1$ & \textbf{0.675} $\pm$ 0.013 & 0.673 $\pm$ 0.016 & 0.359 $\pm$ 0.079\\ 
    & $\varepsilon = 10$ & 0.671 $\pm$ 0.014 & \textbf{0.678} $\pm$ 0.017  & 0.375 $\pm$ 0.056\\ 
    \hline
    
    \end{tabular}

    \vspace{-0.3cm}
    \caption{The classification accuracy of DPAF with different gradient compression strategies.}
    \label{tab: ablation_7}
\end{table}

\begin{table}[hbt!]
    \centering
    \small
    \vspace{-6pt}
    \begin{tabular}{@{}|c|l|cc|c|@{}}
    \hline & \hspace{2ex}$\varepsilon$ &\shortstack{DPAF}  &\shortstack{w/ TS}  \\
    \hline\hline
    \multirow{2}{*}{\shortstack{CelebA-Gender}} & $\varepsilon = 1$ & \textbf{0.802} $\pm$ 0.018 & \textbf{0.802} $\pm$ 0.015 \\
    & $\varepsilon = 10$ & \textbf{0.826} $\pm$ 0.010 & 0.806 $\pm$ 0.025 \\ 

    \hline
    \multirow{2}{*}{\shortstack{CelebA-Hair}}  & $\varepsilon = 1$ & \textbf{0.675} $\pm$ 0.013 & 0.594 $\pm$ 0.031 \\
    & $\varepsilon = 10$ & \textbf{0.671} $\pm$ 0.014 & 0.620 $\pm$ 0.043  \\ 
    \hline
    
    \end{tabular}

    \vspace{-0.3cm}
    \caption{The classification accuracy of DPAF with tempered sigmoid (TS) activation functions.}
    \label{tab: ablation_9}
\end{table}

\subsubsection{Number of Layers for conv1, conv2*, and conv3*}
Using Celeba-Gender and CelebA-Hair as examples, we aim to know which layer configuration will result in better accuracy. As both Celeba-Gender and CelebA-Hair are $64\times 64$, we know that there are at most five layers in total. Note that, in contrast to ordinary GANs, deliberately setting more layers in DPGANs may, in turn, hurt the training result \cite{6979031} because the lengthier gradient will lead to greater information loss, failing the convergence, according to our experience. There are too many configurations to exhaustively examine. Table~\ref{tab: ablation_4} shows the only results of accuracy in the cases where conv1 and conv2* jointly occupy at most four layers. From Table~\ref{tab: ablation_4}, we know that C2-C1-$\times$, C2-C2-$\times$, C3-C1-$\times$, and C1-C3-$\times$ result in better accuracy. Thus, given the above results, we include the consideration of conv3* in Table~\ref{tab: ablation_5}, because Table~\ref{tab: ablation_4} does not consider conv3*. The results in Table~\ref{tab: ablation_5} support our design choice for the canonical implementation of C2-C2-C1 because it outperforms the other settings.

\subsubsection{The Impact of Asymmetry Multiplier $\mu$.}
A larger $\mu$ implies a much larger budget for updating conv2*, given $\varepsilon_3$ for DPAGG. Obviously, the increased $\mu$ raises accuracy because conv2* virtually has more budget. Moreover, Table~\ref{tab: ablation_6} supports our claim in Section~\ref{sec: Discussion} that $\mu$ cannot be arbitrarily increased. The reason is that the increased $\mu$ also leads to a less frequent update of conv2*, which may in turn degrades the utility.


\subsubsection{The Other Techniques in Enhancing Accuracy}
Many techniques have been proposed to reduce the negative impact of DPSGD on model training. We examine three of them to see whether they provide similar benefits to DPAF. 

\vspace{-0.2cm}\paragraph{Pre-Training the Model with Public Data}
The recent development of DP classifiers and DPGANs has witnessed that extra data may help improve the performance of DP models~\cite{unlocking, zhang2018, Tramr2020DifferentiallyPL}. Here, we want to examine whether pre-training the model with public data helps DPAF raise its utility. Here, the common setting in Table~\ref{tab: ablation_10} is that we follow DPAF to train $C$, perform transfer learning, and then train $G$ and $D$ on the CIFAR-10 dataset without considering DP. After that, DPAF is used to train DPGAN with the pre-trained $D$ as $D$ and the randomly initialized parameters as $G$. We additionally train DPGAN completely based on the pre-trained parameters for both $G$ and $D$. One can see from Table~\ref{tab: ablation_10} that the extra data still helps the utility of DPAF.

\vspace{-0.2cm}\paragraph{The Impact of Gradient Compression.}
Gradient compression (GC)~\cite{gradientcompression} is originally proposed to reduce the communication cost in federated learning. The rationale behind gradient compression is that most of the values in the gradient contribute nearly no information on the update. Different from the original case, where GC works on the gradient averaged over the samples in a batch, the canonical implementation of DPAF adopts GC to keep only the top 90\% values of per-sample gradients and then performs the averaging. However, we still want to examine whether GC can help DPSGD. The comparison between the DPAF column and ``w/o GC'' column in Table~\ref{tab: ablation_7} still shows that DPSGD can benefit from GC because the information loss from gradient clipping can be mitigated.

TOPAGG~\cite{datalens} is a modified DPSGD that works on compressed and quantized gradients. The GC in TOPAGG is configurated to keep the top-$k$ values only\footnote{For CelebA-Gender, $k=200$ with $\varepsilon=1$ and $k=3000$ with $\varepsilon=10$. For CelebA-Hair, $k=150$ with $\varepsilon=1$ and $k=200$ with $\varepsilon=10$.}. Nevertheless, TOPAGG gains lower accuracy. This can be explained by considering the design of TOPAGG. In particular, the success of TOPAGG, in essence, relies on training a large number of teacher classifiers\cite{datalens, dp-sinkhorn}. As DPAF does not fit such a requirement, TOPAGG on DPAF does not perform well. 

\vspace{-0.2cm}\paragraph{Tempered Sigmoid Activation Function}
Papernot et al.~~\cite{Papernot2020TemperedSA} find that exploding activations cause the unclipped gradient magnitude to increase and therefore gradient clipping leads to more information loss. Thus, tempered sigmoid (TS)~\cite{Papernot2020TemperedSA}, a family of activation functions, is proposed to replace the conventional activation functions in DPSGD. Table~\ref{tab: ablation_9} shows the results, where hyperbolic tangent ($\tanh$), as a representative of TS, is used to replace the leaky ReLU in our canonical DPAF. In our test, $\tanh$ is used in DPAF, and we can see from Table~\ref{tab: ablation_9} that it, in turn, leads to worse accuracy. This can be explained as follows. First, Papernot et al.~~\cite{Papernot2020TemperedSA} conduct the experiments on DP classifiers only. Whether TS can raise the utility of DPGANs remains unknown. Second, a bounded activation (e.g., sigmoid and $\tanh$) easily causes gradient vanishing and therefore is rarely used in practice. On the contrary, the unbounded activations (e.g., ReLU and leaky ReLU) are more capable of avoiding gradient vanishing~\cite{cdcgan}. conv1 and conv2* in DPAF are updated by DPSGD, but conv3* and FC* are updated by SGD. While TS is beneficial to DPSGD (for conv1 and conv2*) but harmful to SGD (conv2* and FC*). Overall, adopting $\tanh$ in DPAF slightly degrades the utility. 

\section{Conclusion}\label{sec: Conclusion}
Overall, we propose a novel and effective DPGAN, DPAF, which can synthesize high-dimensional image data. Fundamentally different from the prior works, DPAF is featured by the DP feature aggregation in the forward phase, which significantly improves the robustness against noise. In addition, we propose a novel asymmetric training strategy, which determines an ideal batch size. We formally prove the privacy of DPAF. Extensive experiments demonstrate superior performance compared to the previous state-of-the-art methods. 

\clearpage
{\footnotesize \bibliographystyle{plainurl}
\bibliography{sample}}

\clearpage
\input{Appendix.tex}
\end{document}

%% file: Appendix.tex
\section*{Appendix}

\begin{table}[hbt!]
    \centering
    \small
    \begin{tabular}{@{}|c|l|cccccc|@{}}
    \hline & \hspace{2ex}$\varepsilon$   & \shortstack{C1-C1-$\times$}  & \shortstack{C1-C2-$\times$} & \shortstack{C2-C1-$\times$} & \shortstack{C1-C3-$\times$} & \shortstack{C2-C2-$\times$} & \shortstack{C3-C1-$\times$} \\
    \hline\hline
    \multirow{2}{*}{\shortstack{CelebA-Gender}} & $\varepsilon = 1$ & 0.629 $\pm$ 0.040 & 0.737 $\pm$ 0.025  & \textbf{0.824} $\pm$ 0.025 & 0.661 $\pm$ 0.144  & 0.805 $\pm$ 0.021  & 0.811 $\pm$ 0.020 \\
    & $\varepsilon = 10$ & 0.720 $\pm$ 0.045 & 0.733 $\pm$ 0.038  &  0.762 $\pm$ 0.079 & 0.729 $\pm$ 0.032 & \textbf{0.786} $\pm$ 0.018  & 0.751 $\pm$ 0.036 \\ 

    \hline
    \multirow{2}{*}{\shortstack{CelebA-Hair}}  & $\varepsilon = 1$ & 0.423 $\pm$ 0.089  & 0.475 $\pm$ 0.120   &  0.643 $\pm$ 0.014  & \textbf{0.662} $\pm$ 0.019 & 0.519 $\pm$ 0.095  & 0.623 $\pm$ 0.038\\
    & $\varepsilon = 10$ & 0.565 $\pm$ 0.019 & 0.644 $\pm$ 0.024   &  0.639 $\pm$ 0.039  & 0.657 $\pm$ 0.023 & \textbf{0.683} $\pm$ 0.022  & 0.659 $\pm$ 0.019  \\ 
    \hline
    \end{tabular}
    \caption{The classification accuracy of different layer architecture for conv1 and  conv2*.}
    \label{tab: ablation_4}
\end{table}

\begin{table}[hbt!]
\centering
\begin{tabular}{@{}|p{1.2cm}||p{6.8cm}|@{}}
  \hline
  & Description \\
  \hline
   $\mathcal{D}$, $\mathcal{D}'$ & The neighboring data \\
  \hline
   $C$ & The classifier in DAF before transfer learning \\
  \hline
   $G$ & The generator in DAF after transfer learning \\
  \hline
   $D$ & The discriminator in DAF after transfer learning \\
  \hline
   $\mu$ & Asymmetry multiplier \\
  \hline
   $n_{\text{critic}}$ & Number of critic iterations per generator iteration \\
  \hline
   $\epsilon$ & The privacy loss \\
  \hline
   $\delta$ & The probability of violating DP\\
  \hline
   $\sigma^2$ & The variance of Gaussian distribution \\
  \hline
   $M_f$ & The feature extractor (FE) \\
  \hline
   $M_c$ & The label predictor \\
  \hline
   $\theta_f$ & The parameters of the FE\\
  \hline
   $\theta_c$ & The parameters of the label predictor \\
  \hline
   $u$ & The clipping threshold \\
  \hline
   $\text{clip}_u$ & The gradient clipping function with clipping threshold $u$ \\
  \hline
   $w$ & The model parameter \\
  \hline
   $p$ & The size of feature map is $p\times p$ \\
  \hline
   $m$ & The number of feature maps \\
  \hline
   $\mathbb{B}$ & The number of batches \\
  \hline
   IN & The instance normalization \\
  \hline
   SIN & The simplified instance normalization \\
  \hline
   $\mu_{i_1i_2}$ & The mean of feature map $X_{i_1i_2}$ \\
  \hline
   $\sigma_{i_1i_2}^{2}$ & The variance of feature map $X_{i_1i_2}$ \\
  \hline
   $H$ & The height of the feature map \\
  \hline
   $W$ & The width of the feature map \\
  \hline
   $x_{i_1i_2i_3i_4}$ & The element of feature map $X_{i_1i_2}$ \\
  \hline
   $\widehat{x_{i_1i_2i_3i_4}}$ & The new value of $x_{i_1i_2i_3i_4}$ after SIN \\
  \hline
   $\alpha$ & The order in R\'{e}nyi DP \\
  \hline
   $D_{\alpha}$ & The R\'{e}nyi divergence of order $\alpha$ \\
  \hline
   $G_{\sigma}$ & The  Gaussian mechanism with variance $\sigma^2$ \\
  \hline
   $\gamma$ & The subsampling rate \\
  \hline
\end{tabular}
\caption{Notation Table}
\label{table: Notation Table}
\end{table}

\begin{algorithm}[t!]
    \DontPrintSemicolon
	\caption{Training of DPAF}
	\label{algo: Training of DPAF}
	\textbf{Notation}: number of batches $\mathbb{B}$, mean square error loss function $\mathcal{L}_{\text{MSE}}$, binary cross-entropy loss functions $\mathcal{L}_{\text{BCE}}$, $\mathcal{L}'_{\text{BCE}}$, $\mathcal{L}''_{\text{BCE}}$, asymmetry multiplier $\mu$, number of critic iterations per generator iteration $n_{critic}$

    \tcc{the for-loop below trains the classifier $C$}
    \For{$i=1$ to $\mathbb{B}$}
    {
        compute $\mathcal{L}_\text{MSE}$ over
the $i$-th batch 

        SGD for updating conv2, conv3, and FC
        
        DPSGD($\varepsilon_1$) for updating conv1
    }
    conv1*$\leftarrow$ conv1 \tcp*{conv1* and conv1 share parameters}

    \For{$i=1$ to $\mathbb{B}$}
    {
        \tcc{computing $\mathcal{L}_\text{BCE}$ on $D$ with DPAGG($\varepsilon_3$),  $\mathcal{L}_\text{BCE}^{\prime}$ on $D$ that replaces DPAGG($\varepsilon_3$) by AGG, and $\mathcal{L}_\text{BCE}^{\prime\prime}$ on $D$ without DPAGG($\varepsilon_3$), respectively}
        compute $\mathcal{L}_\text{BCE}$ over the $i$-th batch \\
        \tcc{the code below asymmetrically trains $D$}
        \If{$i\%\mu= 0$}
       {
        compute $\mathcal{L}_\text{BCE}^{\prime}$ over the $\left[ i-\mu+1, i \right]$-th batches\\
        SGD for updating conv3* and FC* by $\mathcal{L}_\text{BCE}$\\
        DPSGD($\varepsilon_2$) for updating conv2* by $\mathcal{L}_\text{BCE}^{\prime}$
        }
        \Else
        {
         SGD for updating conv3* and FC* by $\mathcal{L}_\text{BCE}$
        }
        \tcc{the code below trains $G$}
        \If{$i\%n_{critic}= 0$}
        {
            compute $\mathcal{L}_\text{BCE}^{\prime\prime}$ over each sample from the $i$-th batch \\
            SGD for update of $G$
        }
    }
\end{algorithm}

%% file: usenix.bbl
\begin{thebibliography}{10}

\bibitem{Abadi2016DeepLW}
Martín Abadi, Andy Chu, Ian Goodfellow, H.~Brendan McMahan, Ilya Mironov,
  Kunal Talwar, and Li~Zhang.
\newblock Deep learning with differential privacy.
\newblock {\em ACM Conference on Computer and Communications Security (CCS)},
  2016.

\bibitem{wgan}
Martin Arjovsky, Soumith Chintala, and L{\'e}on Bottou.
\newblock {W}asserstein generative adversarial networks.
\newblock In {\em International Conference on Machine Learning (ICML)}, 2017.

\bibitem{Disparate}
Eugene Bagdasaryan, Omid Poursaeed, and Vitaly Shmatikov.
\newblock Differential privacy has disparate impact on model accuracy.
\newblock In {\em Conference on Neural Information Processing Systems
  (NeurIPS)}, 2019.

\bibitem{6979031}
Raef Bassily, Adam Smith, and Abhradeep Thakurta.
\newblock Private empirical risk minimization: Efficient algorithms and tight
  error bounds.
\newblock In {\em IEEE 55th Annual Symposium on Foundations of Computer Science
  (FOCS)}, 2014.

\bibitem{dp-sinkhorn}
Tianshi Cao, Alex Bie, Arash Vahdat, Sanja Fidler, and Karsten Kreis.
\newblock Don't generate me: Training differentially private generative models
  with sinkhorn divergence.
\newblock In {\em Advances in Neural Information Processing Systems (NeurIPS)},
  2021.

\bibitem{isprivatelearning}
Nicholas Carlini, Samuel Deng, Sanjam Garg, Somesh Jha, Saeed Mahloujifar,
  Mohammad Mahmoody, Shuang Song, Abhradeep Thakurta, and Florian Tramer.
\newblock Is private learning possible with instance encoding?
\newblock In {\em IEEE Symposium on Security and Privacy (S\&P)}, 2021.

\bibitem{gs-wgan}
Dingfan Chen, Tribhuvanesh Orekondy, and Mario Fritz.
\newblock Gs-wgan: A gradient-sanitized approach for learning differentially
  private generators.
\newblock In {\em Conference on Neural Information Processing Systems
  (NeurIPS)}, 2020.

\bibitem{gan-leaks}
Dingfan Chen, Ning Yu, Yang Zhang, and Mario Fritz.
\newblock Gan-leaks: A taxonomy of membership inference attacks against
  generative models.
\newblock In {\em ACM Conference on Computer and Communications Security
  (CCS)}, 2020.

\bibitem{dpgen}
Jia-Wei Chen, Chia-Mu Yu, Ching-Chia Kao, Tzai-Wei Pang, and Chun-Shien Lu.
\newblock Dpgen: Differentially private generative energy-guided network for
  natural image synthesis.
\newblock In {\em IEEE/CVF Conference on Computer Vision and Pattern
  Recognition (CVPR)}, 2022.

\bibitem{unlocking}
Soham De, Leonard Berrada, Jamie Hayes, Samuel~L. Smith, and Borja Balle.
\newblock Unlocking high-accuracy differentially private image classification
  through scale.
\newblock arXiv: 2204.13650, 2022.
\newblock URL: \url{https://arxiv.org/pdf/2204.13650.pdf}.

\bibitem{dpdm}
Tim Dockhorn, Tianshi Cao, Arash Vahdat, and Karsten Kreis.
\newblock Differentially private diffusion models.
\newblock arXiv:2210.09929, 2022.
\newblock URL: \url{https://openreview.net/pdf?id=pX21pH4CsNB}.

\bibitem{connect-the-dots}
Vadym Doroshenko, Badih Ghazi, Pritish Kamath, Ravi Kumar, and Pasin
  Manurangsi.
\newblock Connect the dots: Tighter discrete approximations of privacy loss
  distributions.
\newblock In {\em Proceedings on Privacy Enhancing Technologies (PoPETs)},
  2022.

\bibitem{dwork2014algorithmic}
Cynthia Dwork and Aaron Roth.
\newblock The algorithmic foundations of differential privacy.
\newblock {\em Foundations and Trends in Theoretical Computer Science},
  9(3-4):211--407, 2014.

\bibitem{model-inversion-attack}
Matt Fredrikson, Somesh Jha, and Thomas Ristenpart.
\newblock Model inversion attacks that exploit confidence information and basic
  countermeasures.
\newblock In {\em ACM Conference on Computer and Communications Security
  (CCS)}, 2015.

\bibitem{harder2021dpmerf}
Frederik Harder, Kamil Adamczewski, and Mijung Park.
\newblock Dp-merf: Differentially private mean embeddings with random features
  for practical privacy-preserving data generation.
\newblock In {\em International Conference on Artificial Intelligence and
  Statistics (AISTATS)}, 2021.

\bibitem{logan}
Jamie Hayes, Luca Melis, George Danezis, and Emiliano De~Cristofaro.
\newblock Logan: Membership inference attacks against generative models.
\newblock In {\em Proceedings on Privacy Enhancing Technologies (PoPETs)},
  2019.

\bibitem{7780459}
Kaiming He, Xiangyu Zhang, Shaoqing Ren, and Jian Sun.
\newblock Deep residual learning for image recognition.
\newblock In {\em IEEE Conference on Computer Vision and Pattern Recognition
  (CVPR)}, 2016.

\bibitem{MCMCMIA}
Benjamin Hilprecht, Martin H\"{a}rterich, and Daniel Bernau.
\newblock Monte carlo and reconstruction membership inference attacks against
  generative models.
\newblock In {\em Proceedings on Privacy Enhancing Technologies (PoPETs)},
  2019.

\bibitem{differentIN2}
Xun Huang and Serge Belongie.
\newblock Arbitrary style transfer in real-time with adaptive instance
  normalization.
\newblock In {\em International Conference on Computer Vision (ICCV)}, 2017.

\bibitem{dp2-vae}
Dihong Jiang, Guojun Zhang, Mahdi Karami, Xi~Chen, Yunfeng Shao, and Yaoliang
  Yu.
\newblock Dp$^2$-vae: Differentially private pre-trained variational
  autoencoders, 2022.
\newblock URL: \url{https://arxiv.org/abs/2208.03409}.

\bibitem{differentIN1}
Xin Jin, Cuiling Lan, Wenjun Zeng, Zhibo Chen, and Li~Zhang.
\newblock Style normalization and restitution for generalizable person
  re-identification.
\newblock In {\em IEEE/CVF Conference on Computer Vision and Pattern
  Recognition (CVPR)}, 2020.

\bibitem{pate-gan}
James Jordon, Jinsung Yoon, and Mihaela van~der Schaar.
\newblock Pate-gan: Generating synthetic data with differential privacy
  guarantees.
\newblock In {\em International Conference on Learning Representations (ICLR)},
  2019.

\bibitem{pggan}
Tero Karras, Timo Aila, Samuli Laine, and Jaakko Lehtinen.
\newblock Progressive growing of {GAN}s for improved quality, stability, and
  variation.
\newblock In {\em International Conference on Learning Representations (ICLR)},
  2018.

\bibitem{nofreelunch}
Daniel Kifer and Ashwin Machanavajjhala.
\newblock No free lunch in data privacy.
\newblock In {\em ACM SIGMOD International Conference on Management of Data
  (SIGMOD)}, 2011.

\bibitem{Large-Language-Models}
Xuechen Li, Florian Tramèr, Percy Liang, and Tatsunori Hashimoto.
\newblock Large language models can be strong differentially private learners.
\newblock In {\em International Conference on Learning Representations (ICLR)},
  2022.

\bibitem{pearl}
Seng~Pei Liew, Tsubasa Takahashi, and Michihiko Ueno.
\newblock Pearl: Data synthesis via private embeddings and adversarial
  reconstruction learning.
\newblock In {\em International Conference on Learning Representations (ICLR)},
  2022.

\bibitem{gradientcompression}
Yujun Lin, Song Han, Huizi Mao, Yu~Wang, and William~J Dally.
\newblock {Deep Gradient Compression: Reducing the communication bandwidth for
  distributed training}.
\newblock In {\em The International Conference on Learning Representations
  (ICLR)}, 2018.

\bibitem{g-pate}
Yunhui Long, Boxin Wang, Zhuolin Yang, Bhavya Kailkhura, Aston Zhang, Carl~A.
  Gunter, and Bo~Li.
\newblock G-pate: Scalable differentially private data generator via private
  aggregation of teacher discriminators.
\newblock {\em Advances in Neural Information Processing Systems (NeurIPS)},
  2021.

\bibitem{Mckenna2019PGM}
Ryan Mckenna, Daniel Sheldon, and Gerome Miklau.
\newblock Graphical-model based estimation and inference for differential
  privacy.
\newblock In {\em International Conference on Machine Learning (ICML)}, 2019.

\bibitem{McMahan2018AGA}
H.~B. McMahan and G.~Andrew.
\newblock A general approach to adding differential privacy to iterative
  training procedures.
\newblock {\em NeurIPS Workshop on Privacy Preserving Machine Learning (PPML)},
  2018.

\bibitem{RDP}
Ilya Mironov.
\newblock R{\'e}nyi differential privacy.
\newblock {\em 2017 IEEE 30th Computer Security Foundations Symposium (CSF)},
  pages 263--275, 2017.

\bibitem{cgan}
Mehdi Mirza and Simon Osindero.
\newblock Conditional generative adversarial nets, 2014.
\newblock URL: \url{https://arxiv.org/abs/1411.1784}.

\bibitem{nasr2020fakegradient}
Milad Nasr, Reza Shokri, and Amir Houmansadr.
\newblock Improving deep learning with differential privacy using gradient
  encoding and denoising.
\newblock In {\em Theory and Practice of Differential Privacy}, 2020.

\bibitem{acgan}
Augustus Odena, Christopher Olah, and Jonathon Shlens.
\newblock Conditional image synthesis with auxiliary classifier {GAN}s.
\newblock In {\em International Conference on Machine Learning (ICML)}, 2017.

\bibitem{papernot2017semisupervised}
Nicolas Papernot, Martín Abadi, Úlfar Erlingsson, Ian Goodfellow, and Kunal
  Talwar.
\newblock Semi-supervised knowledge transfer for deep learning from private
  training data.
\newblock In {\em International Conference on Learning Representations (ICLR)},
  2017.

\bibitem{papernot2018scalable}
Nicolas Papernot, Shuang Song, Ilya Mironov, Ananth Raghunathan, Kunal Talwar,
  and {\'U}lfar Erlingsson.
\newblock Scalable private learning with pate.
\newblock In {\em International Conference on Learning Representations (ICLR)},
  2018.

\bibitem{Papernot2020TemperedSA}
Nicolas Papernot, Abhradeep Thakurta, Shuang Song, Steve Chien, and {\'U}lfar
  Erlingsson.
\newblock Tempered sigmoid activations for deep learning with differential
  privacy.
\newblock {\em AAAI Conference on Artificial Intelligence (AAAI)}, 2021.

\bibitem{invertibleGAN}
Guim Perarnau, Joost van~de Weijer, Bogdan Raducanu, and Jose~M. Álvarez.
\newblock Invertible conditional gans for image editing.
\newblock In {\em NeurIPS Workshop on Adversarial Training}, 2016.

\bibitem{dpd-fvae}
Bjarne Pfitzner and Bert Arnrich.
\newblock Dpd-fvae: Synthetic data generation using federated variational
  autoencoders with differentially-private decoder, 2022.
\newblock URL: \url{https://arxiv.org/abs/2211.11591}.

\bibitem{cdcgan}
Alec Radford, Luke Metz, and Soumith Chintala.
\newblock Unsupervised representation learning with deep convolutional
  generative adversarial networks.
\newblock In {\em International Conference on Learning Representations (ICLR)},
  2016.

\bibitem{mia}
Reza Shokri, Marco Stronati, Congzheng Song, and Vitaly Shmatikov.
\newblock Membership inference attacks against machine learning models.
\newblock In {\em IEEE Symposium on Security and Privacy (S\&P)}, 2017.

\bibitem{p3gm}
Shun Takagi, Tsubasa Takahashi, Yang Cao, and Masatoshi Yoshikawa.
\newblock P3gm: Private high-dimensional data release via privacy preserving
  phased generative model.
\newblock In {\em IEEE International Conference on Data Engineering (ICDE)},
  2021.

\bibitem{Thakkar2019DifferentiallyPL}
O.~Thakkar, G.~Andrew, and H.~B. McMahan.
\newblock Differentially private learning with adaptive clipping.
\newblock {\em Conference on Neural Information Processing Systems (NeurIPS)},
  2021.

\bibitem{Tramr2020DifferentiallyPL}
Florian Tram{\`e}r and Dan Boneh.
\newblock Differentially private learning needs better features (or much more
  data).
\newblock In {\em International Conference on Learning Representations (ICLR)},
  2021.

\bibitem{1607.08022}
Dmitry Ulyanov, Andrea Vedaldi, and Victor Lempitsky.
\newblock Instance normalization: The missing ingredient for fast stylization,
  2016.
\newblock URL: \url{https://arxiv.org/abs/1607.08022}.

\bibitem{instancenormalization}
Dmitry Ulyanov, Andrea Vedaldi, and Victor Lempitsky.
\newblock Improved texture networks: Maximizing quality and diversity in
  feed-forward stylization and texture synthesis.
\newblock In {\em IEEE Conference on Computer Vision and Pattern Recognition
  (CVPR)}, 2017.

\bibitem{datalens}
Boxin Wang, Fan Wu, Yunhui Long, Luka Rimanic, Ce~Zhang, and Bo~Li.
\newblock Datalens: Scalable privacy preserving training via gradient
  compression and aggregation.
\newblock {\em ACM Conference on Comp`uter and Communications Security (CCS)},
  2021.

\bibitem{Wang2019SubsampledRD}
Yu-Xiang Wang, B.~Balle, and S.~Kasiviswanathan.
\newblock Subsampled r{\'e}nyi differential privacy and analytical moments
  accountant.
\newblock In {\em International Conference on Artificial Intelligence and
  Statistics (AISTATS)}, 2019.

\bibitem{xie2018}
Liyang Xie, Kaixiang Lin, Shu Wang, Fei Wang, and Jiayu Zhou.
\newblock Differentially private generative adversarial network.
\newblock arXiv: 1802.06739, 2018.
\newblock URL: \url{https://arxiv.org/pdf/1802.06739.pdf}.

\bibitem{Xu2019GANobfuscator}
Chungui Xu, Ju~Ren, Deyu Zhang, Yaoxue Zhang, Zhan Qin, and Ren Kui.
\newblock Ganobfuscator: Mitigating information leakage under gan via
  differential privacy.
\newblock {\em IEEE Transactions on Information Forensics and Security},
  14(9):2358--2371, 2019.

\bibitem{removingDisparate}
Depeng Xu, Wei Du, and Xintao Wu.
\newblock Removing disparate impact on model accuracy in differentially private
  stochastic gradient descent.
\newblock In {\em ACM Conference on Knowledge Discovery and Data Mining (KDD)},
  2021.

\bibitem{enhanced-mia}
Jiayuan Ye, Aadyaa Maddi, Sasi~Kumar Murakonda, Vincent Bindschaedler, and Reza
  Shokri.
\newblock Enhanced membership inference attacks against machine learning
  models.
\newblock In {\em ACM Conference on Computer and Communications Security
  (CCS)}, 2022.

\bibitem{Yu2021DoNL}
D.~Yu, Huishuai Zhang, Wei Chen, and T.~Liu.
\newblock Do not let privacy overbill utility: Gradient embedding perturbation
  for private learning.
\newblock {\em International Conference on Learning Representations (ICLR)},
  2021.

\bibitem{yu2022differentially}
Da~Yu, Saurabh Naik, Arturs Backurs, Sivakanth Gopi, Huseyin~A Inan, Gautam
  Kamath, Janardhan Kulkarni, Yin~Tat Lee, Andre Manoel, Lukas Wutschitz,
  Sergey Yekhanin, and Huishuai Zhang.
\newblock Differentially private fine-tuning of language models.
\newblock In {\em International Conference on Learning Representations (ICLR)},
  2022.

\bibitem{Large-scale-private-learning}
Da~Yu, Huishuai Zhang, Wei Chen, Jian Yin, and Tie-Yan Liu.
\newblock Large scale private learning via low-rank reparametrization.
\newblock In {\em International Conference on Machine Learning (ICML)}, 2021.

\bibitem{privbayes}
Jun Zhang, Graham Cormode, Cecilia~M. Procopiuc, Divesh Srivastava, and Xiaokui
  Xiao.
\newblock Privbayes: Private data release via bayesian networks.
\newblock {\em ACM Transactions on Database Systems}, 42(4), oct 2017.

\bibitem{zhang2018}
Xinyang Zhang, Shouling Ji, and Ting Wang.
\newblock Differentially private releasing via deep generative model.
\newblock In {\em arXiv: 1801.01594}, 2018.
\newblock URL: \url{https://arxiv.org/abs/1801.01594}.

\bibitem{zhang2021privsyn}
Zhikun Zhang, Tianhao Wang, Jean Honorio, Ninghui Li, Michael Backes, Shibo He,
  Jiming Chen, and Yang Zhang.
\newblock Privsyn: Differentially private data synthesis.
\newblock In {\em USENIX Security Symposium}, 2021.

\bibitem{dpdc}
Tianhang Zheng and Baochun Li.
\newblock Differentially private dataset condensation, 2023.
\newblock URL: \url{https://openreview.net/pdf?id=H8XpqEkbua_}.

\end{thebibliography}
